\documentclass[10pt,journal,compsoc]{IEEEtran}
\ifCLASSOPTIONcompsoc
  \usepackage[nocompress]{cite}
\else
  \usepackage{cite}
\fi
\ifCLASSINFOpdf
\else
\fi

\usepackage{url}
\usepackage{epsfig}
\usepackage{graphicx}
\usepackage{amsmath}
\usepackage{amssymb}

\usepackage{color}%
\usepackage{cite}%
\usepackage{animate}%
\usepackage{bm}
\usepackage{algorithm}
\usepackage{algorithmic}
\usepackage{multirow}
\usepackage{bbding}
\usepackage{ragged2e}
\usepackage{threeparttable}

\newlength{\itemwidth}

\newtheorem{proposition}{\bf Proposition}
\newtheorem{proof}{\bf Proof}

\usepackage[pagebackref=false,breaklinks=false,colorlinks=false]{hyperref}

\begin{document}
%
\title{Rolling Shutter Inversion: Bring Rolling Shutter Images to High Framerate Global Shutter Video}
%
%
%
%

\author{Bin Fan,~\IEEEmembership{Student Member,~IEEE,}
        Yuchao Dai,~\IEEEmembership{Member,~IEEE,}
        and~Hongdong Li,~\IEEEmembership{Senior Member,~IEEE}
\IEEEcompsocitemizethanks{\IEEEcompsocthanksitem Bin Fan and Yuchao Dai are with School of Electronics and Information, Northwestern Polytechnical University, Xi'an, 710129, China. E-mail: binfan@mail.nwpu.edu.cn, daiyuchao@gmail.com. Yuchao Dai is the corresponding author.
\IEEEcompsocthanksitem Hongdong Li is with School of Computing, the Australian National University, Canberra, Australia. E-mail: hongdong.li@gmail.com.}

}

%
%

\markboth{Journal of \LaTeX\ Class Files,~Vol.~14, No.~8, August~2015}%
{Shell \MakeLowercase{\textit{et al.}}: Bare Demo of IEEEtran.cls for Computer Society Journals}
\IEEEtitleabstractindextext{%
\begin{abstract}

\justifying{
A single rolling-shutter (RS) image may be viewed as a row-wise combination of a sequence of global-shutter (GS) images captured by a (virtual) moving GS camera within the exposure duration. Although rolling-shutter cameras are widely used, the RS effect causes obvious image distortion especially in the presence of fast camera motion, hindering downstream computer vision tasks. In this paper, we propose to {\em invert} the rolling-shutter image capture mechanism, i.e., recovering a continuous high framerate global-shutter video from two time-consecutive RS frames. We call this task the RS temporal super-resolution (RSSR) problem. The RSSR is a very challenging task, and to our knowledge, no practical solution exists to date. This paper presents a novel deep-learning based solution. By leveraging the multi-view geometry relationship of the RS imaging process, our learning based framework successfully achieves high framerate GS generation. Specifically, three novel contributions can be identified: (i) novel formulations for bidirectional RS undistortion flows under constant velocity as well as constant acceleration motion model. (ii) a simple linear scaling operation, which bridges the RS undistortion flow and regular optical flow. (iii) a new mutual conversion scheme between varying RS undistortion flows that correspond to different scanlines. Our method also exploits the underlying spatial-temporal geometric relationships within a deep learning framework, where no additional supervision is required beyond the necessary middle-scanline GS image. Building upon these contributions, this paper represents the very first rolling-shutter temporal super-resolution deep-network that is able to recover high framerate global-shutter videos from just two RS frames. Extensive experimental results on both synthetic and real data show that our proposed method can produce high-quality GS image sequences with rich details, outperforming the state-of-the-art methods.


}

\end{abstract}

}

\maketitle

\IEEEdisplaynontitleabstractindextext

%
\IEEEpeerreviewmaketitle

\IEEEraisesectionheading{\section{Introduction}\label{sec:introduction}}

\IEEEPARstart{M}{ost} consumer-grade cameras are built upon CMOS sensors using a rolling-shutter (RS) mechanism. Unlike its global-shutter (GS) counterpart, the rolling-shutter camera captures an image frame in a sequential row-by-row manner, which invariably leads to the so-called RS effect (\emph{e.g.}, stretching, wobbling) in the obtained images and videos when the camera moves relative to the scene.
This RS effect is becoming a nuisance in photography and in downstream computer vision tasks. Simply ignoring the RS effect often results in performance degradation or even failure for many real-world applications \cite{dai2016rolling,lao2020rolling,albl2019rolling,wang2021depth,im2018accurate}.

Fortunately, we made an important observation in this paper. Namely, we recognize that a single RS image actually captures richer temporal dynamics information about the underlying relative motion between the scene and the camera. This dynamic information, hidden in the RS scanlines, has not been properly utilized by almost all previous works on RS cameras. By exploiting this hidden-and-previously overlooked information, one could possibly remove the RS artifact, and this is the key insight of this paper.


Recovering the above dynamic information from a single RS image, while doable in theory, is extremely challenging in practice due to the severe ill-posedness of the task (\emph{e.g.}, \cite{rengarajan2017unrolling,zhuang2019learning,rengarajan2016bows,lao2018robust}). On the other hand, using multiple RS frames (\emph{e.g.}, \cite{vasu2018occlusion,zhong2021rscd,hedborg2012rolling,fan2021rsstereo}) would make the task more tractable, but at the cost of increased computation. In this paper, we take a middle ground, using a minimum of two consecutive rolling-shutter frames for the task, in order to strike a balance between recovery performance and efficiency (see also \cite{liu2020deep,zhuang2017rolling,zhuang2020homography,fan2021sunet}).

Intuitively, given a pair of consecutive rolling-shutter frames, humans seem to be able to infer a plausible explanation for the underlying geometry and temporal dynamics (\emph{i.e.}, camera motion and scene structure). Hence, recovering the underlying geometry from consecutive rolling-shutter images seems to be possible.
Towards this goal, this paper advocates recovering a continuous temporal sequence of latent global-shutter frames from two consecutive RS images, \emph{i.e.}, this task requires solving and subsequently inverting the underlying rolling-shutter geometry.

\begin{figure*}
	\centering
	\begin{center}
		\setlength{\itemwidth}{17.85cm}
		\includegraphics[width=\textwidth]{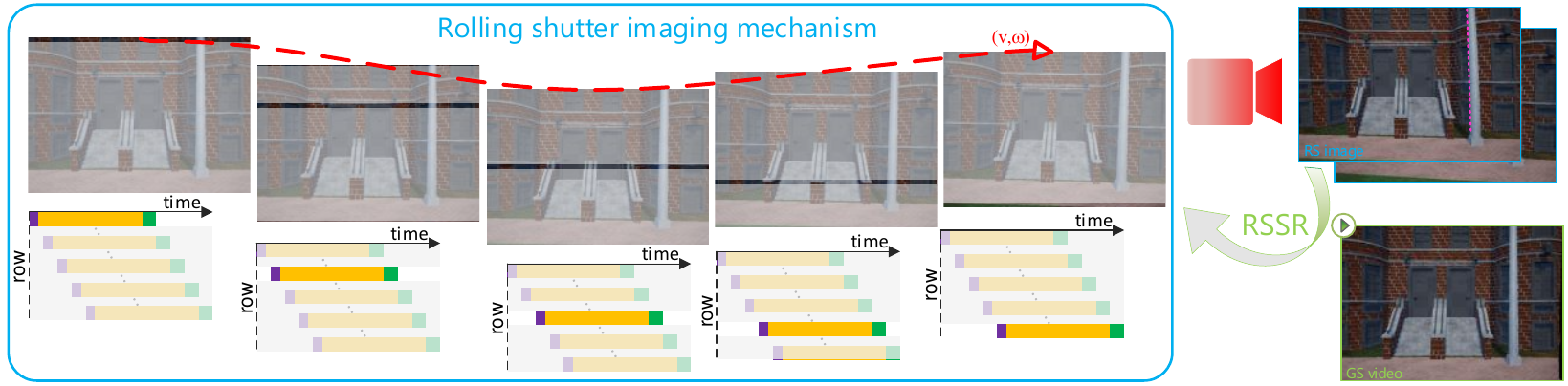}\vspace{-0.50cm}
	\end{center}
	\caption{A rolling-shutter image may be generated by continuously extracting scanlines from continuous global-shutter images row by row and one at a time. In contrast, our RSSR pipeline essentially inverts this process, \emph{i.e.}, reconstructing the latent GS image sequence, frame by frame, from two consecutive RS images. \textit{The recovered GS video (shown in the bottom right position of the figure) can be viewed in the arXiv version of \cite{fan2021rssr}.} \label{fig:Innovation_example_fig1}}
	\vspace{-3.5mm}
\end{figure*}

To the best of our knowledge, this task (which we call RSSR, stands for Rolling-Shutter temporal Super-Resolution) is a novel computer vision task, and has not been studied before. Moreover, existing RS correction methods (\emph{e.g.}, \cite{zhuang2017rolling}) are unable to solve this task, because of the need to estimate small and subtle intra-frame motions, and the need for non-trivial camera calibration and sophisticated optimization procedures. So far, we have not seen any deep-learning based solution for this task either.


In this paper, we propose a novel deep-learning based method to {\em invert} the rolling-shutter image capture mechanism, \emph{i.e.}, which recovers a continuous global-shutter video sequence from two consecutive rolling-shutter images. The concept of our method is illustrated in Fig.~\ref{fig:Innovation_example_fig1}. This inversion process resembles the aforementioned human ability of inferring temporal dynamics from consecutive RS frames.

In the literature, the following two classes of existing works are mostly related to our method: (i) GS frame interpolation (\emph{e.g.}, \cite{jiang2018super,niklaus2020softmax,bao2019depth}) and (ii) two-image RS correction (\emph{e.g.}, \cite{liu2020deep,zhuang2017rolling,zhuang2020homography}). However, we would like to point out that our new method is significantly different from both classes of methods. For example, our method achieves RS-aware pixel displacement by scaling the corresponding optical flow vector under the constant velocity motion model, whereas the GS frame interpolation methods often rely on small and controllable pixel displacement and hence fall short when dealing with the RS effect. For another example, our method is free from the unsmoothness or the ghosting artifact which often hinders the traditional RS correction methods.
Although existing deep-learning based RS correction methods have achieved impressive performance, by using them only one single GS frame can be recovered (\emph{e.g.}, DeepUnrollNet \cite{liu2020deep}, or See Subsection~\ref{Analysis_DeepUnrollNet} for more examples). In contrast, our method produces a visually fluid GS video sequence.

Our method does more than RS correction. Besides eliminating geometric distortion, we also recover a set of high framerate GS images in order, for example, to recover 960 GS images from two 480-height RS images to get a GS video with 480$\times$ frame rate up-conversion. This is particularly challenging as one has to ensure the temporal smoothness of the recovered video sequence.

We formulate the bidirectional RS undistortion flows to characterize the pixel-wise RS-aware pixel displacement, and further advance two calculation methods for the mutual conversion between varying RS undistortion flows corresponding to different scanlines, \emph{i.e.}, constant velocity propagation and constant acceleration propagation. To this end, we prove that the scaling factor is in the interval of $(-1,1)$ when correcting an RS image to its middle-scanline GS image. As a result, we propose a data-driven solution for RSSR with good interpretability, which intrinsically encapsulates the underlying RS geometry that more sophisticated methods (\emph{e.g.}, \cite{zhuang2017rolling,zhuang2020homography}) struggle to learn. Also, we propose to estimate the camera acceleration to improve the accuracy and adaptation of GS video sequence recovery.

Our proposed geometry-aware RSSR pipeline employs a cascaded architecture to extract a latent high framerate GS video sequence from two consecutive RS images.
Firstly, we estimate the bidirectional optical flows by using the classic PWC-Net \cite{sun2018pwc}. Secondly, we design an encoder-decoder UNet network to learn the scaling factor of each pixel (\emph{i.e.}, the middle-scanline correlation map in Subsection~\ref{sec:Connection_UF_OF}) such that the middle-scanline RS undistortion flows can be inferred. Meanwhile, the RS undistortion flows for any scanline can be associated and propagated explicitly. Finally, the softmax splatting \cite{niklaus2020softmax} is used to produce the high framerate GS video frames at arbitrary scanlines.
Our vanilla RSSR network can be trained end-to-end and only the middle-scanline GS images are needed for supervision.
Additionally, the first-scanline GS images can be leveraged as a supervisor to estimate the camera acceleration, thereby obtaining better RS inversion results.
Since none of the learned network parameters are time-dependent, it can synthesize as many GS frames as needed, \emph{i.e.}, yielding a GS video with an arbitrary framerate.
It is also worth mentioning that our approach has a very high efficiency of RS inversion due to the geometry-aware propagation.
Extensive experiment results on benchmark datasets demonstrate that
our approach not only outperforms state-of-the-art methods in both RS effect removal and inference efficiency, but also produces a smooth and coherent video sequence surpassing all previous methods.

A preliminary version of this paper appeared in \cite{fan2021rssr}. In this extended version,
1) we extend the formulation to a more realistic constant acceleration motion model and then propose an alternative but promising Acceleration-Net to further recover higher quality GS video frames; 2) we provide more experimental analyses on DeepUnrollNet \cite{liu2020deep}; 3) we demonstrate the advantages of our approach over the state-of-the-art video frame interpolation methods and the two-stage methods; 4) we provide a more comprehensive generalization and robustness analysis; 5) we conduct a more in-depth comparison with the analytic solution in \cite{zhuang2017rolling}; 6) we confirm the effectiveness of our approach for 3D reconstruction. Note that the extension of our constant acceleration propagation is also geometrically driven.

To summarize, our main contributions are as follows:
\begin{itemize}
	\item We give a detailed proof to the scanline-dependent bidirectional RS undistortion flows. 
	\vspace{0.25mm}
	\item We propose the first learning-based RSSR solution for latent global-shutter video sequence recovery. 
	\vspace{0.25mm}
	\item We develop an effective Acceleration-Net for constant acceleration propagation, enhancing the generality and accuracy of constant velocity propagation in reverting rolling-shutter cameras.
	\vspace{0.25mm}
\end{itemize}


\section{Related Work} \label{sec:relatedWork}
\noindent\textbf{Video frame interpolation (VFI).}
After extensive research in recent years, VFI can be roughly classified into phase-based \cite{meyer2015phase,meyer2018phasenet}, kernel-based \cite{choi2020channel,liu2017video} and flow-based \cite{park2020bmbc,jiang2018super,niklaus2020softmax} approaches.
Notably, with advances in optical flow estimation \cite{sun2018pwc,teed2020raft}, flow-based VFI methods have been actively adopted to explicitly exploit motion information.
Not only limited to linear flow interpolation \cite{jiang2018super,niklaus2018context,niklaus2020softmax}, many variants are dedicated to better intermediate flow estimation, such as quadratic \cite{xu2019quadratic}, rectified quadratic \cite{liu2020enhanced}, and cubic \cite{chi2020all} approximations.
Moreover, the initial flow field is enhanced in DAIN \cite{bao2019depth} via a depth-aware flow projection layer. The symmetric motion field is constructed in BMBC \cite{park2020bmbc} to generate the intermediate flow directly.
On the other side, more attention is paid to detail refinement and fusion, including contextual warping \cite{niklaus2018context,bao2019depth}, occlusion inference \cite{xue2019video,jiang2018super}, cycle constraints \cite{reda2019unsupervised,liu2019deep}, joint deblurring and frame rate up-conversion \cite{shen2020blurry}, and softmax splatting \cite{niklaus2020softmax} for more efficient forward warping, \emph{etc.}
Unfortunately, these VFI methods are tailored to GS cameras, which cannot correctly synthesize in-between frames for RS inputs.
In contrast, our method can not only efficiently interpolate scanline-arbitrary intermediate frames, but also can successfully remove the RS artifacts.


\vspace{0.25mm}
\noindent\textbf{Traditional non-learning RS correction.}
Over the last decade, several works have revisited the RS geometric model to remove the RS effect \cite{lao2020rolling,rengarajan2016bows,forssen2010rectifying,hedborg2012rolling,im2018accurate,saurer2013rolling,wang2020relative,fan2021rsdpsnet,purkait2017rolling,lao2021solving}.
Grundmann \emph{et al.} \cite{grundmann2012calibration} employed a homography mixture to achieve joint RS removal and video stabilization.
The occlusion-aware undistortion method \cite{vasu2018occlusion} removed the depth-dependent RS distortions from a specific setting of $\ge3$ RS images, assuming a piece-wise planar 3D scene.
Zhuang \emph{et al.} \cite{zhuang2017rolling} proposed estimating the full camera motion by a differential formulation to remove RS distortions in two consecutive RS images. Such a model has achieved considerable success, but further improvements have appeared challenging, due to the difficulties of making it robust and efficient to various situations (\emph{e.g.}, relying too much on the initial optical flow estimation \cite{zhuang2020homography}).
Subsequently, Zhuang and Tran \cite{zhuang2020homography} put forward a differential RS homography to model the scanline-varying RS camera poses, which can be used to perform RS-aware image stitching and correction.
Fan \emph{et al.} \cite{fan2021rsstereo} extended \cite{zhuang2017rolling} and presented an RS-stereo-aware differential epipolar constraint to correct the RS stereo image pairs.
Albl \emph{et al.} \cite{albl2020two} explored a simple two-camera rig with different shutter directions to undistort the RS images acquired by a smartphone.
Wu \emph{et al.} \cite{wu2021simultaneous} dealt with video stabilization and RS removal jointly.

\vspace{0.25mm}
\noindent\textbf{Deep-learning based RS correction.}
Recently, convolutional neural networks (CNN) have been used to achieve more flexible and efficient RS correction.
Rengarajan \emph{et al.} \cite{rengarajan2017unrolling} proposed the first CNN to correct a single RS image by assuming a simple affine model.
Zhuang \emph{et al.} \cite{zhuang2019learning} extended \cite{rengarajan2017unrolling} to learn the RS geometry from a single frame through two independent networks, yielding geometrically more correct results.
Zhong \emph{et al.} \cite{zhong2021rscd} simultaneously handled RS correction and deblurring.
Liu \emph{et al.} \cite{liu2020deep} and Fan \emph{et al.} \cite{fan2021sunet} used two consecutive RS images as input and designed specialized CNNs to predict one GS image.
To the best of our knowledge, our RSSR model is the first that is developed to learn the mapping from two consecutive RS frames to high framerate GS video frames corresponding to any scanline.

\vspace{-0.65mm}
\section{Differential Forward RS Geometry} \label{sec:preliminaries}
\noindent\textbf{GS-aware forward warping.}
Assuming that the GS camera experiences constant linear velocity $\mathbf{v}=[v_1, v_2, v_3]^{\text T}$ and angular velocity $\bm{\omega}=[\omega_1, \omega_2, \omega_3]^{\text T}$ over two consecutive frames 1 and 2, a 3D point $\mathbf{X}$ with depth $Z$ is observed by the camera to move with the 3D velocity as $-{\mathbf v}-{\bm \omega} \times {\mathbf X}$ \cite{longuet1980interpretation}. Projecting this 3D velocity into the 2D image plane yields the image motion field ${\mathbf f}$, which is usually approximated by the optical flow vector $({\mathbf f}_u,{\mathbf f}_v)^{\text T}$ under the brightness constancy assumption, at pixel ${\mathbf x}=(x,y)$ as \cite{longuet1980interpretation}:
\begin{equation}\label{eq:1}
{\mathbf f} = \frac{{\mathbf A} {\mathbf v}}{Z} + {\mathbf B} {\bm \omega}  \buildrel \Delta \over=  \pi({\mathbf v},{\bm \omega},{\mathbf x},Z,f),
\end{equation}
where
\begin{equation}\label{eq:2}
\begin{split}
{\mathbf A} &= \left[\begin{array}{ccc}
-{f} & 0 & x \\
0 & -{f} & y
\end{array}\right],\\
{\mathbf B} &= \left[\begin{array}{ccc}
\frac{x y}{{f}} & -\left(f+\frac{x^{2}}{{f}}\right) & y \\
\left(f+\frac{y^{2}}{{f}}\right) & -\frac{x y}{{f}} & -x
\end{array}\right].
\end{split}
\end{equation}
Here, $(x,y)$ is the normalized image coordinate and $f$ denotes the focal length.

\vspace{0.5mm}
\noindent\textbf{RS forward motion parameterization.}
Since all RS scanlines are successively and instantaneously exposed one by one at different times, each scanline possesses a different optical center.
To account for the scanline-varying camera poses, we make full use of the small motion hypothesis.
Given the inter-frame camera velocities $({\mathbf v},{\bm \omega})$ between the two first scanlines of two consecutive RS images, the subtle intra-frame relative motion could be obtained by interpolation \cite{patron2015spline,saurer2013rolling,zhuang2017rolling,zhuang2020homography}.
Specifically, \cite{zhuang2017rolling} derived the linear interpolation scheme under the assumption of constant velocity motion and the quadratic interpolation scheme under the assumption of constant acceleration motion.
Formally, the camera position and rotation $({\mathbf p}_1^{s_1},{\mathbf r}_1^{s_1})$ (resp. $({\mathbf p}_2^{s_2},{\mathbf r}_2^{s_2})$) of the $s_1$-th (resp. $s_2$-th) scanline in frame 1 (resp. 2) \emph{w.r.t.} the first scanline of frame 1 can be expressed as:
\begin{equation}\label{eq:3}
\begin{array}{ll}
\mathbf{p}_{1}^{s_{1}}=\lambda_{1}^{s_{1}} \mathbf{v}, \,\,& \mathbf{r}_{1}^{s_{1}}=\lambda_{1}^{s_{1}} \bm{\omega}, \\
\mathbf{p}_{2}^{s_{2}}=\lambda_{2}^{s_{2}} \mathbf{v}, & \mathbf{r}_{2}^{s_{2}}=\lambda_{2}^{s_{2}} \bm{\omega},
\end{array}
\end{equation}
where
\begin{equation}\label{eq:4}
\left\{
\begin{aligned}
\lambda_{1}^{s_{1}}&=\frac{\gamma s_{1}}{h} \\
\lambda_{2}^{s_{2}}&=1+\frac{\gamma s_{2}}{h}
\end{aligned}
\right.,
\end{equation}
are for the constant velocity motion assumption, or
\begin{equation}\label{eq:5}
\left\{
\begin{aligned}
\lambda_{1}^{s_{1}}&=\frac{2}{k+2}\left(\frac{\gamma s_{1}}{h}+\frac{k}{2} \left(\frac{\gamma s_{1}}{h}\right)^{2}\right) \\
\lambda_{2}^{s_{2}}&=\frac{2}{k+2}\left(1+\frac{\gamma s_{2}}{h}+\frac{k}{2}\left(1+\frac{\gamma s_{2}}{h}\right)^{2}\right)
\end{aligned}
\right.,
\end{equation}
are for the constant acceleration motion assumption.
Note that here, $\gamma={h\tau_r}/{{\tau}}$ represents the readout time ratio \cite{zhuang2017rolling}, \emph{i.e.}, the ratio between the total readout time and the total time $\tau$ in an RS frame, which can be calibrated by \cite{meingast2005geometric,oth2013rolling}.
$\tau_r$ is the readout time of a single scanline and $h$ is the total number of scanlines in an image.
$k$ represents an extra unknown motion parameter describing the acceleration, denoted as the acceleration factor ($k>0$ for acceleration and $k<0$ for deceleration). It is easy to verify that Eq.~\eqref{eq:5} will reduce to Eq.~\eqref{eq:4} when the acceleration factor vanishes, \emph{i.e.}, $k=0$.
Consequently, the relative motion between $s_1$-th and $s_2$-th scanlines satisfies:
\begin{equation}\label{eq:6}
\begin{aligned}
\mathbf{v}_{s_{1} s_{2}} &=\mathbf{p}_{2}^{s_{2}}-\mathbf{p}_{1}^{s_{1}}=\left(\lambda_{2}^{s_{2}}-\lambda_{1}^{s_{1}}\right) \mathbf{v}, \\
\bm{\omega}_{s_{1} s_{2}} &=\mathbf{r}_{2}^{s_{2}}-\mathbf{r}_{1}^{s_{1}}=\left(\lambda_{2}^{s_{2}}-\lambda_{1}^{s_{1}}\right) \bm{\omega}.
\end{aligned}
\end{equation}
Further details can be found in \cite{zhuang2017rolling}.

\vspace{0.5mm}
\noindent\textbf{RS-aware forward warping.}
As ${\mathbf f}_v=s_2-s_1$, plugging Eqs.~\eqref{eq:4} and \eqref{eq:6} into Eq.~\eqref{eq:1}, we can relate the forward optical flow of pixel ${\mathbf x}$ in frame 1 in terms of the RS camera velocity and the 3D point as:
\begin{equation}\label{eq:7}
\left[ {\begin{array}{*{20}{c}}
	{\mathbf f}_u\\
	{\mathbf f}_v
	\end{array}} \right]  = \alpha \left[ {\begin{array}{*{20}{c}}
	\pi_u({\mathbf v},{\bm \omega},{\mathbf x},Z,f)\\
	\pi_v({\mathbf v},{\bm \omega},{\mathbf x},Z,f)
	\end{array}} \right],
\end{equation}
where
\begin{equation}\label{eq:8}
\alpha = 1 + \frac{\gamma {\mathbf f}_v}{h}
\end{equation}
represents the RS-aware forward interpolation factor under the constant velocity motion model, which depends on the scanline involved in the optical flow. $\pi_u(\cdot)$ and $\pi_v(\cdot)$ denotes the first and second entries of $\pi(\cdot)$, respectively.
Hence, the geometric inaccuracies induced by the RS effect could be compensated by simply scaling the optical flow vector of each pixel ${\mathbf x}$.
Note that this paper does not investigate the RS-aware optical flow modeling under the assumption of constant acceleration motion due to its complexity.

\vspace{-0.5mm}
\section{RS Undistortion Flow vs. Optical Flow} \label{sec:flows_versus}
\subsection{RS imaging mechanism}
As shown in Fig.~\ref{fig:Innovation_example_fig1}, the RS camera exposes each scanline in sequence, which results in a different local frame for each scanline.
The RS image $\mathbf{I}_{r}$ can therefore be regarded as the result of successive row-by-row combinations of virtual GS image sequences over the period of camera readout time, \emph{i.e.}, formulating the RS imaging model as:
\begin{equation}\label{eq:9}
\left\lfloor\mathbf{I}_{r}(\mathbf{x})\right\rfloor_{s}=\left\lfloor\mathbf{I}^{s}_{g}(\mathbf{x})\right\rfloor_{s}, \,\,\, 0\le s \le h-1,
\end{equation}
where $\mathbf{I}^{s}_{g}$ is the virtual GS image captured at time $s\tau_r$. $\lfloor\cdot\rfloor_s$ indicates extracting the $s$-th scanline.
On the contrary, the RSSR task aims to reverse the above RS imaging formulation, \emph{i.e.}, estimating the displacement vectors ${\mathbf u}_{r \to s}$ of pixel ${\mathbf x}$ from the RS image to the virtual GS image at $s$-th scanline such that
\begin{equation}\label{eq:10}
\mathbf{I}_{r}(\mathbf{x})=\mathbf{I}^{s}_{g}(\mathbf{x}+\mathbf{u}_{r \to s}), \,\,\, 0\le s \le h-1.
\end{equation}
Specifically, it is to estimate the dense \emph{RS undistortion flow} ${\mathbf U}_{r \to s}$ that corresponds to scanline $s$ by stacking ${\mathbf u}_{r \to s}$ for all pixels in matrix form. Note that similar to \cite{liu2020deep}, ${\mathbf U}_{r \to s}$ is a forward warping operation, \emph{e.g.}, softmax splatting \cite{niklaus2020softmax}.

\begin{figure}[!t]
	\centering
	\includegraphics[width=0.42\textwidth]{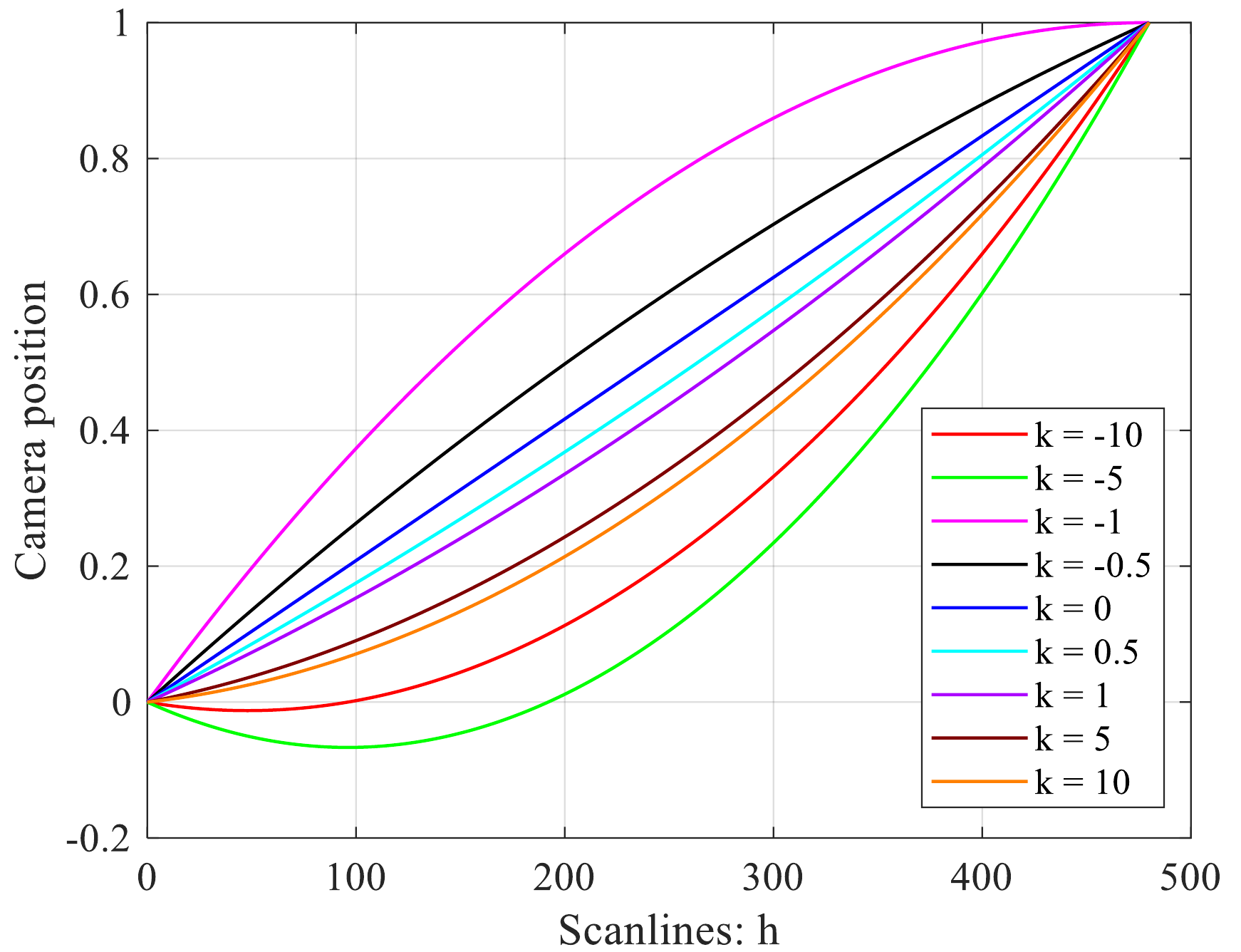}\vspace{-2.5mm}
	\caption{Generality of scanline pose interpolation given the more realistic constant acceleration motion assumption. Here, similar to \cite{zhuang2020homography}, we assume $\gamma=1$, $h=480$, a unit displacement length between two first scanlines, and different values of acceleration factor $k$. \label{fig:fig_acceleration_model}}
	\vspace{-1.5mm}
\end{figure}

\subsection{Bidirectional RS undistortion flows}
To deliver each RS pixel ${\mathbf x}$ on the $\eta$-th scanline in frame 1 to its GS canvas defined by the pose corresponding to $s$-th scanline of frame 1, the RS-aware forward warping displacement vector of pixel ${\mathbf x}$ can be formulated as:
\begin{equation}\label{eq:11}
\left[ {\begin{array}{*{20}{c}}
	{\mathbf u}_u\\
	{\mathbf u}_v
	\end{array}} \right] = \beta \left[ {\begin{array}{*{20}{c}}
	\pi_u({\mathbf v},{\bm \omega},{\mathbf x},Z,f)\\
	\pi_v({\mathbf v},{\bm \omega},{\mathbf x},Z,f)
	\end{array}} \right],
\end{equation}
where
\begin{equation}\label{eq:12}
\beta = \frac{\gamma (s-\eta)}{h}
\end{equation}
represents the RS-aware forward undistortion factor under the constant velocity motion model, which depends on the scanline offset between the target scanline $s$ and the current scanline $\eta$. Note that $0 \le s,\eta \le h-1$ and the pixels on the $s$-th scanline of the RS image remain unchanged. Stacking the forward warping displacement vectors of all pixels in frame 1, we can obtain the pixel-wise forward RS undistortion flow $\mathbf U_{1 \to s}$ of frame 1, which could be used to restore the latent GS image corresponding to any scanline $s \in [0,h-1]$ of frame 1.

Unfortunately, Eqs.~\eqref{eq:7} and \eqref{eq:11} describe merely the forward optical flow (\emph{i.e.}, from frame 1 to frame 2) and the forward RS undistortion flow (\emph{i.e.}, from frame 1 to its scanline $s$), respectively. Given two input consecutive RS images, one cannot remove the geometric RS distortion in frame 2 if only using the RS-aware forward warping alone.
Therefore, we further propose an \emph{RS-aware backward warping model} accounting for frame 2, as shown in Proposition~\ref{p1}.
\begin{proposition}\label{p1}
The RS-aware backward warping, \emph{i.e.}, the backward optical flow from frame 2 to frame 1 and the backward RS undistortion flow from frame 2 to its scanline $s$, can be modeled by simply taking a negative readout time ratio under the constant velocity motion assumption.
\end{proposition}

\begin{proof}
To formulate the RS-aware backward warping accounting for frame 2, the backward inter-frame camera velocities $({\mathbf v}',{\bm \omega}')$ should obey:
${\mathbf v}' =-{\mathbf v}$ and ${\bm \omega}' =-{\bm \omega}$.
Let $Z'$ denote the depth of each pixel ${\mathbf x}'$ in frame 2 and $({\mathbf f}'_u,{\mathbf f}'_v)$ the backward optical flow vector from frame 2 to frame 1. In complete analogy with the RS forward motion parameterizations in Section~\ref{sec:preliminaries}, we again derive the relative motion between the scanline $s_1$ of frame 1 and the scanline $s_2$ of frame 2 as:
\begin{equation}\label{eq:s2}
\begin{aligned}
\mathbf{v}_{s_{2} s_{1}} &=\left(\lambda_{2}^{s_{2}}-\lambda_{1}^{s_{1}}\right) \mathbf{v}', \\
\bm{\omega}_{s_{2} s_{1}} &=\left(\lambda_{2}^{s_{2}}-\lambda_{1}^{s_{1}}\right) \bm{\omega}',
\end{aligned}
\end{equation}
where $\lambda_{1}^{s_{1}}$ and $\lambda_{1}^{s_{2}}$ satisfy Eq.~\eqref{eq:4} as well.
Note that ${\mathbf f}'_v=s_1-s_2$. In the same way, we can obtain RS-aware backward warping model for the backward optical flow at pixel ${\mathbf x}'$ as:
\begin{equation}\label{eq:s4}
\left[ {\begin{array}{*{20}{c}}
	{\mathbf f}'_u\\
	{\mathbf f}'_v
	\end{array}} \right]  = \alpha' \left[ {\begin{array}{*{20}{c}}
	\pi_u({\mathbf v}',{\bm \omega}',{\mathbf x}',Z',f)\\
	\pi_v({\mathbf v}',{\bm \omega}',{\mathbf x}',Z',f)
	\end{array}} \right],
\end{equation}
where
\begin{equation}\label{eq:s5}
\alpha' = 1 - \frac{\gamma {\mathbf f}'_v}{h}
\end{equation}
represents the RS-aware backward interpolation factor under the constant velocity motion model.

Furthermore, we derive the RS-aware backward warping displacement vector, which transforms each RS pixel ${\mathbf x}'$ on $\eta$-th scanline of frame 2 to obtain a distortion-free frame defined by the pose of the $s$-th scanline of frame 2 by
\begin{equation}\label{eq:s6}
\left[ {\begin{array}{*{20}{c}}
	{\mathbf u}'_u\\
	{\mathbf u}'_v
	\end{array}} \right] = \beta' \left[ {\begin{array}{*{20}{c}}
	\pi_u({\mathbf v}',{\bm \omega}',{\mathbf x}',Z',f)\\
	\pi_v({\mathbf v}',{\bm \omega}',{\mathbf x}',Z',f)
	\end{array}} \right],
\end{equation}
where
\begin{equation}\label{eq:s7}
\beta' = -\frac{\gamma (s-\eta)}{h}
\end{equation}
represents the RS-aware backward undistortion factor under the constant velocity motion model. Similarly, the GS image corresponding to any scanline $s \in [0,h-1]$ of frame 2 can be recovered.

\end{proof}

Heretofore, we have obtained bidirectional RS undistortion flows $\mathbf U_{1 \to s}$ and $\mathbf U_{2 \to s}$ under the assumption of constant velocity motion, which can warp two consecutive RS images to their GS counterparts corresponding to their respective scanlines $s \in \left[0,h-1\right]$, \emph{i.e.}, activating the rolling-shutter temporal super-resolution.
However, despite its simplicity in compensating for inter-frame motion, as shown in Fig.~\ref{fig:fig_acceleration_model}, a more realistic constant acceleration motion assumption might enhance the generality of the RSSR task in practice. To this end, for the constant acceleration motion model, we utilize Eqs.~\eqref{eq:3} and ~\eqref{eq:5} and rewrite the bidirectional RS-aware undistortion factor in Eq.~\eqref{eq:12} as:
\begin{equation}\label{eq:13}
\beta = \frac{\gamma (s-\eta)}{h} \cdot \frac{2h+k\gamma (s-\eta)}{h(k+2)}.
\end{equation}

Note that due to the concise and well-defined characteristics of constant velocity modeling, we will mainly focus on it for further analysis. Note also that the constant acceleration formulation in this paper is mainly used to further advance the accuracy and compatibility of mutual conversion between varying RS undistortion flows.

\begin{figure}[!t]
	\centering
	\includegraphics[width=0.4872\textwidth]{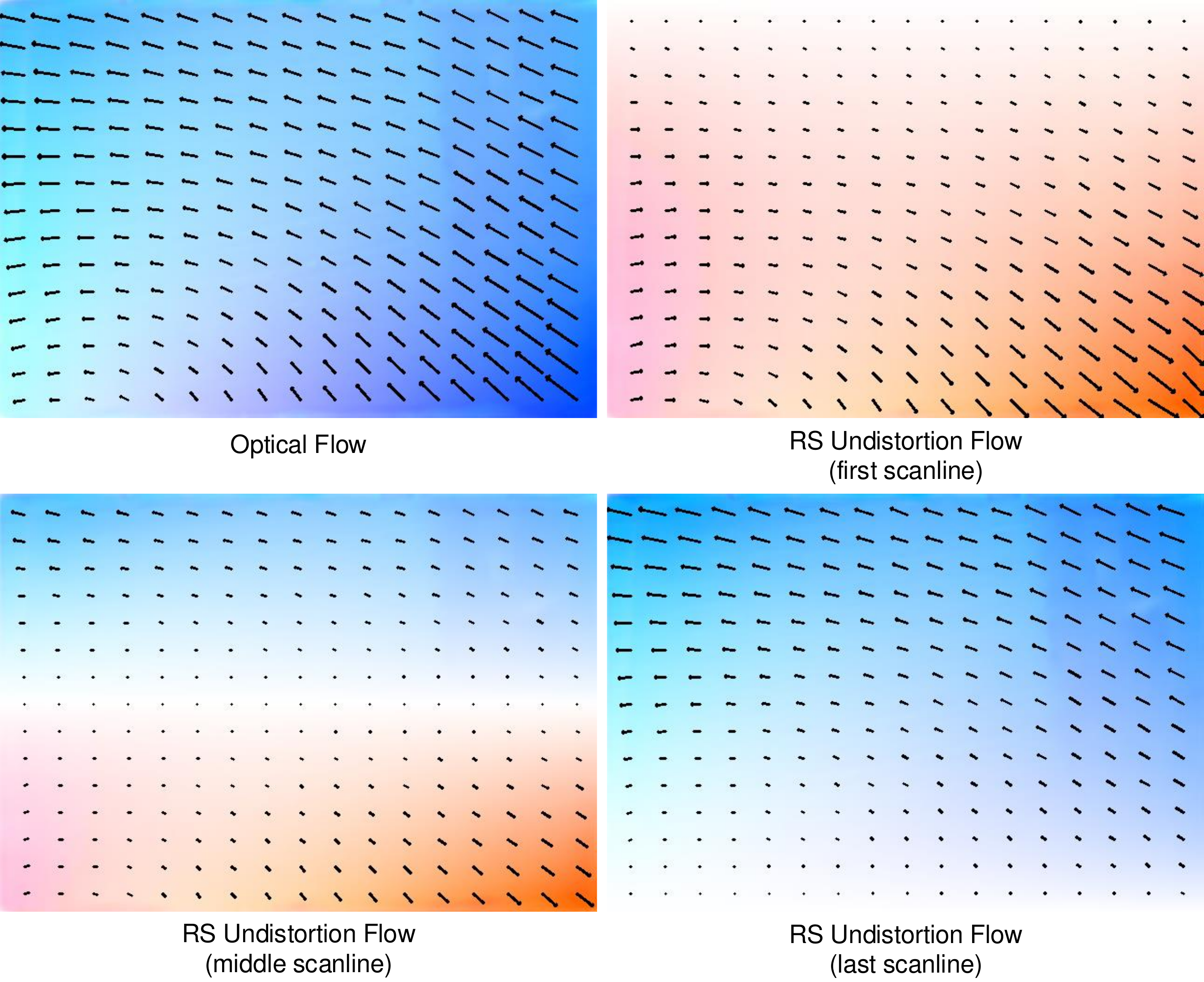}\vspace{-1.0mm}
	\caption{RS undistortion flow versus regular optical flow. Here, forward flows are visualized according to \cite{sun2010secrets}. Compared to the isotropically smooth optical flow map, the RS undistortion flow map exhibits a more significant scanline dependence. On the one hand, the RS undistortion flows near the target scanline appear as lighter colors (\emph{i.e.}, smaller warping displacement values). On the other hand, the RS undistortion flows corresponding to pixels smaller than and larger than the target scanline show different colors (\emph{i.e.}, different warping displacement directions).\label{fig:UF_VS_OF}}
	\vspace{-1.0mm}
\end{figure}
\subsection{Connection between RS undistortion flow and optical flow}\label{sec:Connection_UF_OF}
Note that the optical flow in Eq.~\eqref{eq:7} exhibits the pixel displacement over two consecutive RS frames, while the RS undistortion flow in Eq.~\eqref{eq:11} models the pixel displacement between the RS frame 1 (or frame 2) and the GS frame at scanline $s$.
Without loss of generality, taking the forward warping as an example (note that the backward warping is similar, except that $\gamma$ is negative).
First, eliminating ${\mathbf f}_v$ on the right hand side of Eq.~\eqref{eq:7}, the optical flow of pixel ${\mathbf x}$ can be rewritten as:
\begin{equation}\label{eq:14}
\left[ {\begin{array}{*{20}{c}}
	{\mathbf f}_u\\
	{\mathbf f}_v
	\end{array}} \right]  = \frac{h}{h-\gamma \pi_v} \left[ {\begin{array}{*{20}{c}}
	\pi_u\\
	\pi_v
	\end{array}} \right],
\end{equation}
where $\pi_u$ and $\pi_v$ are abbreviations, determined by camera parameters, camera motions and 3D depths.
Then, through Eqs.~\eqref{eq:11}, ~\eqref{eq:12} and ~\eqref{eq:14}, we establish the connection between the forward/backward RS undistortion flow $({\mathbf u}_u,{\mathbf u}_v)^{\text T}$ and the forward/backward optical flow $({\mathbf f}_u,{\mathbf f}_v)^{\text T}$ at pixel $\mathbf x$ as:
\begin{equation}\label{eq:15}
\left[ {\begin{array}{*{20}{c}}
	{\mathbf u}_u\\
	{\mathbf u}_v
	\end{array}} \right] = c \left[ {\begin{array}{*{20}{c}}
	{\mathbf f}_u\\
	{\mathbf f}_v
	\end{array}} \right],
\end{equation}
where
\begin{equation}\label{eq:16}
c = \frac{\gamma (s-\eta) (h-\gamma\pi_v)}{h^2}
\end{equation}
denotes the forward/backward correlation factor between bidirectional RS undistortion flow and optical flow, which are distinguished by the sign of $\gamma$. Stacking $c$ for all pixels yields the forward and backward correlation maps ${\mathbf C}_{1 \to s}$ and ${\mathbf C}_{2 \to s}$, respectively.
The correlation map thus indicates the underlying RS geometry, \emph{i.e.}, camera motion and scene structure.
Note that the connection under the constant acceleration motion is not discussed in this paper.

Consequently, after obtaining the bidirectional optical flows between two consecutive RS frames (\emph{e.g.} via PWC-Net \cite{sun2018pwc} or RAFT \cite{teed2020raft}), we can scale them to obtain the bidirectional RS undistortion flows to further recover the latent GS images corresponding to specific scanline $s \in [0,h-1]$.
Note that the size and sign of the scaling factor (\emph{i.e.}, correlation factor $c$) rely on the underlying RS geometry.
As illustrated in Fig.~\ref{fig:UF_VS_OF}, we can intuitively observe that, compared to the isotropically smooth optical flow map, the RS undistortion flow is more typically scanline-dependent. The closer the pixel to scanline $s$ (assuming to correction to $s$-th scanline, \emph{e.g.}, the first scanline), the smaller the value of the RS undistortion flow is generally.
Moreover, we point out that Eq.~\eqref{eq:15} is essentially the core of RS correction in \cite{zhuang2017rolling}, but it requires sophisticated processing (such as RANSAC \cite{fischler1981random}, non-convex optimization, \emph{etc.}) to estimate accurate camera motion and scene depth based on the optical flow map. Nevertheless, potential gross errors in optical flow estimates could lead to severe artifacts in the texture-less or non-overlapping regions \cite{zhuang2020homography}, as elaborated in Subsection~\ref{Comparison_RSC}.

\begin{proposition}\label{p2}
	Given bidirectional optical flows ${\mathbf F}_{1 \to 2}$ and ${\mathbf F}_{2 \to 1}$, to recover the GS image corresponding to the camera pose of the middle scanline, the forward correlation map ${\mathbf C}_{1 \to m}$ (resp. backward correlation map ${\mathbf C}_{2 \to m}$) for the forward RS undistortion flow ${\mathbf U}_{1 \to m}$ (resp. backward RS undistortion flow ${\mathbf U}_{2 \to m}$) in Eq.~\eqref{eq:15} satisfies:
	\begin{equation} \label{eq:19}
	\left\lfloor{\mathbf C}_{1 \to m}\right\rfloor_{i,j} \in \left\{ \begin{array}{ll}
	(0,1) & \textrm{if $i<\frac h2$}\\
	0 & \textrm{if $i=\frac h2$}\\
	(-1,0) & \textrm{if $i>\frac h2$}\\
	\end{array} \right. ,
	\end{equation}
	resp.
	\begin{equation} \label{eq:20}
	\left\lfloor{\mathbf C}_{2 \to m}\right\rfloor_{i,j} \in \left\{ \begin{array}{ll}
	(-1,0) & \textrm{if $i<\frac h2$}\\
	0 & \textrm{if $i=\frac h2$}\\
	(0,1) & \textrm{if $i>\frac h2$}\\
	\end{array} \right. ,
	\end{equation}
	where $\lfloor\cdot\rfloor_{i,j}$ is the entry in $i$-th row and $j$-th column.
\end{proposition}

\begin{proof}
	We first decompose $c_m$ defined in Eq.~\eqref{eq:16} into
	\begin{equation}\label{eq:21}
	c_m = \frac{s-\eta}{h} \cdot \frac{\gamma (h-\gamma\pi_v)}{h} \buildrel \Delta \over= c_m^a \cdot c_m^b.
	\end{equation}
	Taking the case of $i<\frac h2$ in Eq.~\eqref{eq:19} as an example, \emph{i.e.}, $0\le \eta<\frac h2$.
	Since the GS image corresponding to the middle scanline is to be restored, $s=\frac h2$. Thus,
	$0 < c_m^a \le \frac 12$.
	In practice, $\pi_v$ indicates the prediction of the inter-frame vertical optical flow value (see Eq.~\eqref{eq:1}), which is usually much smaller than the total number of image scanlines $h$. Also, the readout time ratio $\gamma$ is less than or equal to 1 for a normal RS camera \cite{zhuang2017rolling,im2018accurate,ringaby2012efficient}, so $\gamma \in (0,1]$ when modeling the RS-aware forward warping. We thus arrive at $0< c_m^b < 2$. To sum up, $c_m \in \left(0,1\right)$, \emph{i.e.}, $\left\lfloor{\mathbf C}_{1 \to m}\right\rfloor_{i,j} \in \left(0,1\right)$ holds when $i<\frac h2$. Other cases can be proved similarly. Note that $\gamma \in [-1,0)$ when modeling the RS-aware backward warping.
\end{proof}

\subsection{Interconversion between varying RS undistortion flows}
Eq.~\eqref{eq:11} defines the RS undistortion flow that warps the RS frame to its GS counterpart corresponding to scanline $s$.
Assuming that two GS images corresponding to the $s_1$-th scanline and the $s_2$-th scanline are to be restored, it can be seen from Eqs.~\eqref{eq:11} and \eqref{eq:12} that the difference in the warping displacement vector of the same pixel ${\mathbf x}$ is only $(s_1-\eta)$ and $(s_2-\eta)$ under the constant velocity motion model. Note that ${\mathbf x}$ lies on scanline $\eta$. As a consequence, we can simply extend the RS undistortion flow of pixel ${\mathbf x}$ facing scanline $s_1$ to that for scanline $s_2$ by
\begin{equation}\label{eq:17}
\left[ {\begin{array}{*{20}{c}}
	{\mathbf u}_u^{s_2}\\
	{\mathbf u}_v^{s_2}
	\end{array}} \right] = \frac{s_2-\eta}{s_1-\eta} \left[ {\begin{array}{*{20}{c}}
	{\mathbf u}_u^{s_1}\\
	{\mathbf u}_v^{s_1}
	\end{array}} \right].
\end{equation}
Analogously for constant acceleration motion model, based on Eqs.~\eqref{eq:11} and \eqref{eq:13}, we can obtain:
\begin{equation}\label{eq:18}
\left[ {\begin{array}{*{20}{c}}
	{\mathbf u}_u^{s_2}\\
	{\mathbf u}_v^{s_2}
	\end{array}} \right] = \frac{(s_2-\eta)(2h+\varphi (s_2-\eta))}{(s_1-\eta)(2h+\varphi (s_1-\eta))} \left[ {\begin{array}{*{20}{c}}
	{\mathbf u}_u^{s_1}\\
	{\mathbf u}_v^{s_1}
	\end{array}} \right],
\end{equation}
where $\varphi \buildrel \Delta \over= k\gamma$ denotes an auxiliary variable that needs to be additionally estimated.
Note that the constant velocity propagation by Eq.~\eqref{eq:17} is a special case of the constant acceleration propagation by Eq.~\eqref{eq:18} at $\varphi=0$ (\emph{i.e.}, $k=0$).
At this point, it can be propagated explicitly by either Eq.~\eqref{eq:17} or Eq.~\eqref{eq:18}, which is differentiable, to obtain the RS undistortion flows corresponding to any scanline.

In other words, the resulting RS undistortion flows corresponding to different scanlines can be converted to each other, which is also the key to achieve our RSSR solution.
An example is shown in Fig.~\ref{fig:UF_VS_OF}.
One can find that the difference between varying RS undistortion flows is not only reflected in the distribution of the warping displacement size, but also the warping displacement direction is closely related to the scanline.

\begin{figure}[!t]
	\centering
	\includegraphics[width=0.493\textwidth]{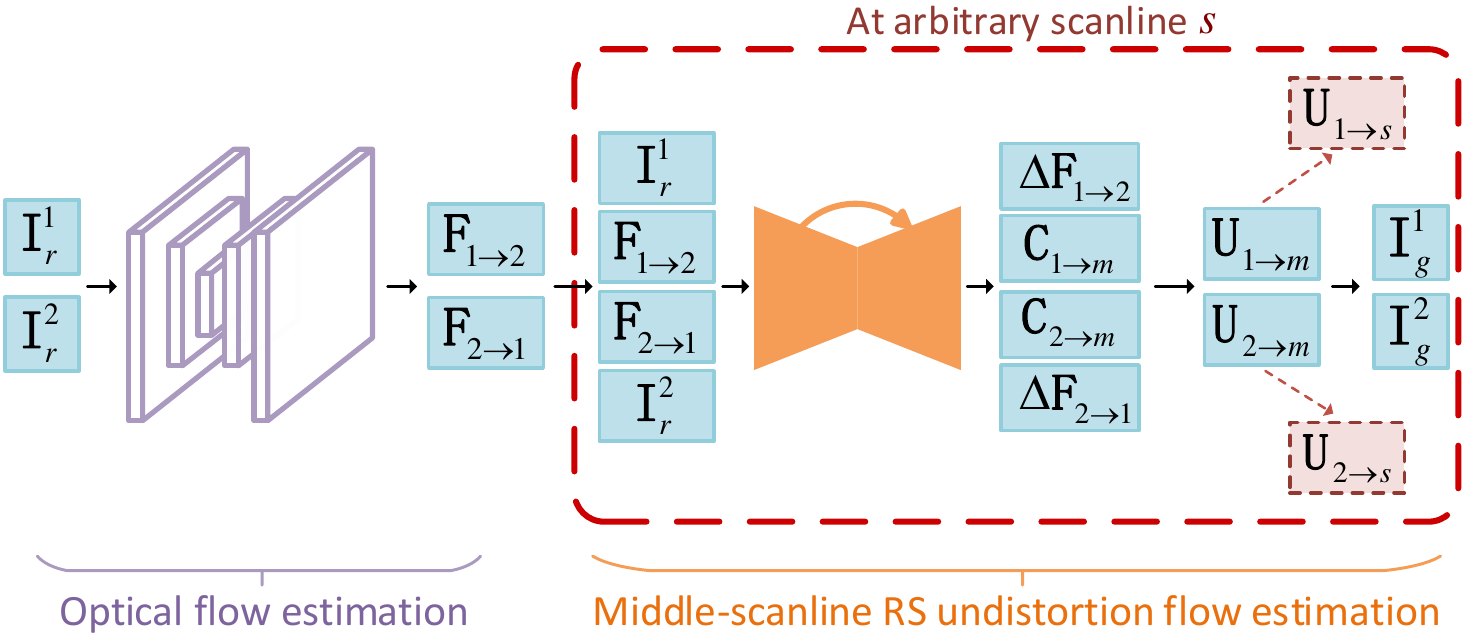}\vspace{-1.0mm}
	\caption{Overview of our RSSR network architecture. Given two consecutive RS frames, we first estimate the bidirectional optical flows. Then, we utilize an encoder-decoder UNet architecture with skip connections, which is associated with Eq.~\eqref{eq:22}, to resolve the correlation maps. Next, the middle-scanline RS undistortion flows can be explicitly calculated by Eq.~\eqref{eq:23}, while being certifiable. Finally, we adopt softmax splatting to generate the target middle-scanline GS frames. Note that our main network is devised to predict the latent GS images corresponding to the middle scanline during training. In particular, in the test phase, the RS undistortion flows for any scanline $s\in[0,h-1]$ can be propagated through Eq.~\eqref{eq:17} or Eq.~\eqref{eq:18} (\emph{cf.}, dashed arrow), followed by the recovery of the global-shutter image corresponding to scanline $s$.\label{fig:pipeline}}
	\vspace{-2.0mm}
\end{figure}

\section{RS Temporal Super-Resolution} \label{sec:RSSR}
The objective of RS temporal super-resolution (RSSR) is to invert the rolling-shutter imaging mechanism, \emph{i.e.}, recovering a high framerate GS video from two consecutive RS images.
In this section, we will utilize the aforementioned theoretical formulations, which reveal the intrinsic geometric properties of the RS inversion task, to elaborate a concise and effective end-to-end CNN to exploit the inherent regularity of data.
Inspired by the cascaded architecture for video frame interpolation in \cite{jiang2018super,niklaus2020softmax,bao2019depth}, we design a dedicated RSSR network to encapsulate the complete underlying RS geometry, as illustrated in Fig.~\ref{fig:pipeline}.

Our network takes two consecutive RS images ${\mathbf I}_r^1$ and ${\mathbf I}_r^2$ as input and predicts the GS image corresponding to any scanline (\emph{i.e.}, generating a high framerate GS video).
Our proposed vanilla RSSR pipeline can be distilled down to two main submodules: the optical flow estimation network $\mathcal{F}$ and the middle-scanline RS undistortion flow estimation network $\mathcal{U}$.
We first utilize $\mathcal{F}$ to obtain bidirectional optical flows, and then encode the relation between optical flows and middle-scanline RS undistortion flows by the middle-scanline correlation maps over $\mathcal{U}$. Finally, we compute the middle-scanline RS undistortion flows to produce two target middle-scanline GS frames by the softmax splatting \cite{niklaus2020softmax}.
Below we describe each sub-network in detail.

\vspace{0.5mm}
\noindent\textbf{Optical flow estimator.}
We employ the classical PWC-Net \cite{sun2018pwc} as our optical flow estimation network $\mathcal{F}$. As learning optical flow without ground truth (GT) supervision is extremely difficult, we fine-tune $\mathcal{F}$ on the RS benchmarks from the pre-trained model of PWC-Net in a self-supervised manner \cite{wang2018occlusion,liu2020self}.

\begin{figure}[!t]
	\centering
	\includegraphics[width=0.488\textwidth]{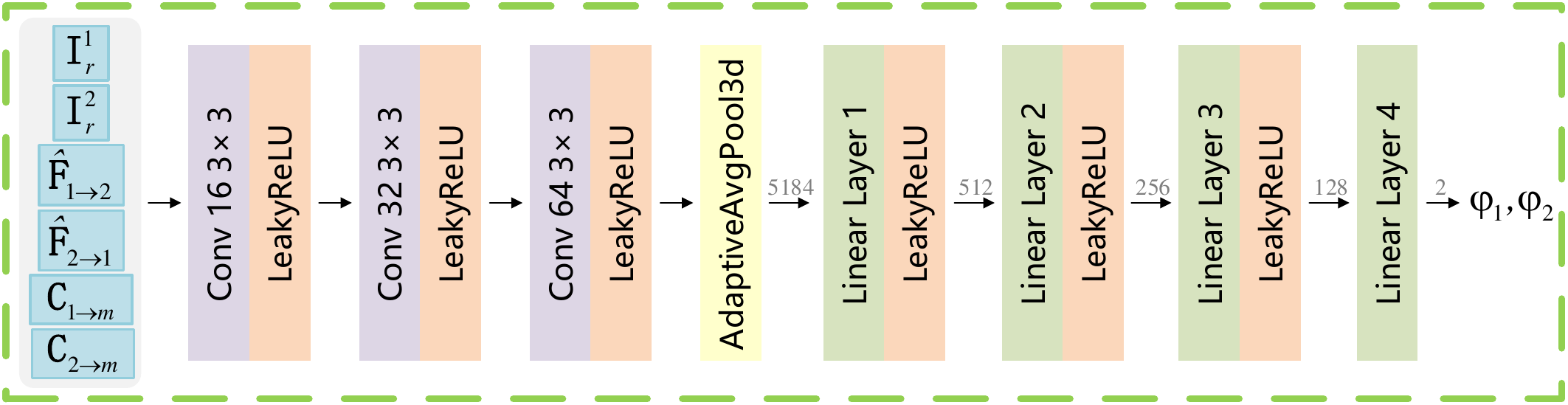}
	\caption{Architecture of our Acceleration-Net $\mathcal{A}$, which corresponds to Eq.~\eqref{eq:18}, \emph{i.e.}, formulating the dashed arrow in Fig.~\ref{fig:pipeline}. To improve the generality and accuracy of constant velocity propagation, $\mathcal{A}$ learns to predict the parameters of constant acceleration propagation ($\varphi_1, \varphi_2$ in particular). \label{fig:acceleration_net}}
	\vspace{-0.40mm}
\end{figure}

\vspace{0.5mm}
\noindent\textbf{Middle-scanline RS undistortion flow estimator.}
We adopt an encoder-decoder UNet architecture \cite{ronneberger2015unet,jiang2018super} with skip connections as our backbone network $\mathcal{U}$ to predict the correlation maps, modifying the last layer to output 6-channel predictions.
To comply with Proposition \ref{p2}, we fetch the first two channels in the output of network $\mathcal{U}$, followed by a \emph{Sigmoid} operation to map to the interval of $(0,1)$, resulting in $\hat{\mathbf C}_{1 \to m}$ and $\hat{\mathbf C}_{2 \to m}$.
Afterward, inspired by \cite{liu2020deep}, we multiply them with their respective normalized scanline offsets (\emph{i.e.}, $\mathbf{T}_{1 \to m}$ and $\mathbf{T}_{2 \to m}$, $\left\lfloor{\mathbf T}\right\rfloor_{i,j} \in [-1,1]$), which is defined as the normalized scanline offset between the captured pixel and that of the middle scanline.
Note that ${\mathbf T}_{1 \to m}=-{\mathbf T}_{2 \to m}$, and their values in middle-scanline are set to be small enough (not to be zero) for subsequent flow propagation.
The final correlation maps ${\mathbf C}_{1 \to m}$ and ${\mathbf C}_{2 \to m}$ thus can be recovered as:
\begin{equation} \label{eq:22}
\begin{array}{ll}
\mathbf{C}_{1 \to m}={\mathbf T}_{1 \to m} \odot \hat{\mathbf C}_{1 \to m}, \\
\mathbf{C}_{2 \to m}={\mathbf T}_{2 \to m} \odot \hat{\mathbf C}_{2 \to m},
\end{array}
\end{equation}
where $\odot$ is an element-wise multiplier. It is easy to verify that they all naturally fit for the theoretical bounds of Proposition \ref{p2}.
Furthermore, we use the last four channels of network $\mathcal{U}$ to estimate optical flow residuals $\Delta\mathbf{F}_{1 \to 2}$ and $\Delta\mathbf{F}_{2 \to 1}$ to enhance the alignment of edge details and the generality of the proposed model.
Finally, according to Eq.~\eqref{eq:15}, the bidirectional RS undistortion flows $\mathbf{U}_{1 \to m}$ and $\mathbf{U}_{2 \to m}$ can be obtained as:
\begin{equation} \label{eq:23}
\begin{array}{ll}
\mathbf{U}_{1 \to m}={\mathbf C}_{1 \to m} \odot \left(\mathbf{F}_{1 \to 2}+\Delta\mathbf{F}_{1 \to 2}\right), \\
\mathbf{U}_{2 \to m}={\mathbf C}_{2 \to m} \odot \left(\mathbf{F}_{2 \to 1}+\Delta\mathbf{F}_{2 \to 1}\right).
\end{array}
\end{equation}

Finally, the softmax splatting \cite{niklaus2020softmax} is used to warp the RS image $\mathbf{I}_r^1$ (resp. $\mathbf{I}_r^2$) to the target middle-scanline GS image $\mathbf{I}_g^1$ (resp. $\mathbf{I}_g^2$) through $\mathbf{U}_{1 \to m}$ (resp. $\mathbf{U}_{2 \to m}$).

\vspace{0.5mm}
\noindent\textbf{Extension to GS images under arbitrary scanlines.}
As the resulting RS undistortion flows corresponding to different scanlines can be converted to each other, based on Eq.~\eqref{eq:17} directly or on Eq.~\eqref{eq:18} with additional estimation of $k$, we can propagate the middle-scanline RS undistortion flows obtained above to reach the RS undistortion flows $\mathbf{U}_{1 \to s}$ and $\mathbf{U}_{2 \to s}$ for any scanline $s\in[0,h-1]$. Then, the corresponding high framerate GS images can be warped from the original two RS images, in the same vein, by using softmax splatting, \emph{i.e.}, achieving the rolling-shutter temporal super-resolution in a temporally coherent manner.

As an alternative, we further devise a compact and lightweight Acceleration-Net module to efficiently estimate the variable $\varphi$ in order to make Eq.~\eqref{eq:18} work for maintaining more realistic and accurate temporal coherence across corrected GS video frames.
Acceleration-Net $\mathcal{A}$ is constructed and executed after the trained optical flow estimation network $\mathcal{F}$ and middle-scanline RS undistortion flow estimation network $\mathcal{U}$.
Subsequently, as shown in Fig.~\ref{fig:acceleration_net}, the generated maps
are cascaded and fed into three $3\times3$ convolutional layers (each followed by a LeakyReLU activation layer) for extracting features, then a 3D adaptive average pooling layer together with four fully-connected layers is used to predict $\varphi_1$ and $\varphi_2$ in forward and backward warpings, respectively.
At this time, we freeze the parameters of $\mathcal{F}$ and $\mathcal{U}$, and train $\mathcal{A}$ separately by using the GT first-scanline GS images for supervision.
Note that our vanilla RSSR pipeline can work well with a constant velocity propagation, where the GT first-scanline GS image is not required, \emph{i.e.}, our entire pipeline is flexible and tractable.

\vspace{0.5mm}
\noindent\textbf{Loss function.}
Given a pair of consecutive RS images ${\mathbf I}_r^1$ and ${\mathbf I}_r^2$, our vanilla RSSR network predicts the bidirectional optical flows ${\mathbf F}_{1 \to 2}$ and ${\mathbf F}_{2 \to 1}$, the bidirectional RS undistortion flows ${\mathbf U}_{1 \to m}$ and ${\mathbf U}_{2 \to m}$, and the target middle-scanline GS images ${\mathbf I}_g^1$ and ${\mathbf I}_g^2$.
Note that when training Acceleration-Net separately, we continue to use ${\mathbf I}_g^1$ and ${\mathbf I}_g^2$ to refer specifically to the target first-scanline GS images for brevity.
Meanwhile, ${\mathbf I}_{gt}^1$ and ${\mathbf I}_{gt}^2$ denote their corresponding GT middle-scanline or first-scanline GS images.
Our overall loss function $\mathcal{L}$ is a linear combination of the reconstruction loss $\mathcal{L}_r$, perceptual loss $\mathcal{L}_p$ \cite{johnson2016perceptual}, warping loss $\mathcal{L}_w$, and smoothness loss $\mathcal{L}_s$:
\begin{equation}\label{eq:24}
\vspace{-0.4mm}
\mathcal{L} = \mu_r\mathcal{L}_r + \mu_p\mathcal{L}_p + \mu_w\mathcal{L}_w + \mu_s\mathcal{L}_s,
\vspace{-0.4mm}
\end{equation}
where $\mu_r$, $\mu_p$, $\mu_w$ and $\mu_s$ are hyper-parameters to balance different losses.
Note that when training the vanilla RSSR network, we set $\mu_r=10$, $\mu_p=1$, $\mu_w=10$, and $\mu_s=0.1$; when training Acceleration-Net $\mathcal{A}$ based on the pre-trained $\mathcal{F}$ and $\mathcal{U}$, we remove the warping loss $\mathcal{L}_w$ and the smoothness loss $\mathcal{L}_s$ while keeping the rest the same, \emph{i.e.}, $\mu_w=\mu_s=0$.
In the following, we describe in detail how to supervise the training of the vanilla RSSR network.

\emph{Reconstruction loss $\mathcal{L}_r$}: We measure the pixel-wise reconstruction qualities of the corrected middle-scanline GS images as:
\begin{equation}\label{eq:25}
\vspace{-0.4mm}
\mathcal{L}_{r} = \sum_{i=1}^{2}\left\|{\mathbf I}_{g}^{i}-{\mathbf I}_{gt}^{i}\right\|_{1}.
\vspace{-0.4mm}
\end{equation}

\emph{Perceptual loss $\mathcal{L}_p$}: To mitigate the blur in the corrected middle-scanline GS images, we use a perceptual loss $\mathcal{L}_p$ \cite{johnson2016perceptual} to preserve details of the predictions and make estimated GS images sharper. Similar to \cite{liu2020deep}, $\mathcal{L}_p$ is defined as:
\begin{equation}\label{eq:26}
\vspace{-0.4mm}
\mathcal{L}_{p} = \sum_{i=1}^{2}\left\| \phi \left( {\mathbf I}_{g}^{i} \right) - \phi \left( {\mathbf I}_{gt}^{i} \right) \right\|_{1},
\vspace{-0.4mm}
\end{equation}
where $\phi$ represents the $conv3\_3$ feature extractor of the VGG19 model \cite{simonyan2014very}.

\emph{Warping loss $\mathcal{L}_w$}: Besides supervising the GS image predictions, we also introduce the warping loss $\mathcal{L}_w$ to maintain the qualities of final bidirectional optical flows, defined as:
\begin{equation}\label{eq:27}
\small
\mathcal{L}_{w} = \left\| {\mathbf I}_{r}^{1} - g( {\mathbf I}_{r}^{2}, \hat{\mathbf F}_{1 \to 2} ) \right\|_{1} + \left\| {\mathbf I}_{r}^{2} - g( {\mathbf I}_{r}^{1}, \hat{\mathbf F}_{2 \to 1} ) \right\|_{1},
\end{equation}
where $\hat{\mathbf F}_{1 \to 2}={\mathbf F}_{1 \to 2}+\Delta\mathbf{F}_{1 \to 2}$, $\hat{\mathbf F}_{2 \to 1}={\mathbf F}_{2 \to 1}+\Delta\mathbf{F}_{2 \to 1}$, and $g(\cdot,\cdot)$ is the backward warping function.

\emph{Smoothness loss $\mathcal{L}_s$}: At last, a smoothness term \cite{liu2017video} is employed to enforce the smoothness of the bidirectional optical flows and bidirectional RS undistortion flows as:
\begin{equation}\label{eq:28}
\mathcal{L}_{s} = \sum_{i=1}^{2}\sum_{j=1,j \not = i}^{2} \left\|\nabla \hat{\mathbf F}_{i \to j}\right\|_{2} + \left\|\nabla {\mathbf U}_{i \to m}\right\|_{2}.
\end{equation}

\section{Experiments} \label{sec:experiments}
\subsection{Implementation Details} \label{sec:Details}
\noindent\textbf{Dataset.}
We train and evaluate our method on the Carla-RS and Fastec-RS datasets \cite{liu2020deep} that provide ground truth middle-scanline and first-scanline GS supervisory signals.
The Carla-RS dataset is generated from a virtual 3D environment based on the Carla simulator \cite{dosovitskiy2017carla}, involving general six degrees of freedom motions.
The Fastec-RS dataset contains real-world RS images synthesized by a professional high-speed GS camera.
The image is normalized to $640\times448$ resolution to fit the basic input requirements of the backbone network PWC-Net.
Following \cite{liu2020deep}, the Carla-RS dataset is divided into a training set of 210 sequences and a test set of 40 sequences, and the Fastec-RS dataset has 56 sequences for training and the remaining 20 ones for testing.
There are no overlapping scenes between the training and test data.
We train our vanilla RSSR network on both two benchmarks to predict the middle-scanline GS images (\emph{i.e.}, $s=h/2$). At the test time, our model is able to be extended to generate arbitrary GS video frames for any scanline $s\in[0,h-1]$.

\vspace{0.5mm}
\noindent\textbf{Training strategy.}\label{Training_strategy}
Our pipeline\footnote{Code can be available at \url{https://github.com/GitCVfb/RSSR}.} is implemented in PyTorch. We use Adam \cite{Kingma_Adam_ICLR_2015} as the optimizer. The learning rate is initially set to $10^{-4}$ and decreases by a factor of 0.8 every 50 epochs. The optical flow estimation network $\mathcal{F}$ is first fine-tuned for 100 epochs from the pre-trained PWC-Net model \cite{sun2018pwc}, and then the entire model is jointly trained for another 200 epochs. Acceleration-Net is subsequently trained for 30 epochs with a fixed learning rate of $10^{-4}$. The training batch size is set to 6. We use a uniform random crop at a horizontal resolution of 256 pixels for data augmentation. Note that we do not change the longitudinal resolution to warrant the scanline dependence of RS images.

\subsection{Evaluation Protocols}
We compute the average Peak Signal-to-Noise Ratio (PSNR), Structural Similarity Index (SSIM), and Learned Perceptual Image Patch Similarity (LPIPS) \cite{zhang2018unreasonable} between estimated and GT GS images.
Higher PSNR/SSIM or lower LPIPS scores indicate better performance.
Note that \emph{unless otherwise stated, all competing methods refer to the GS image corresponding to the first scanline of the second RS frame} for consistent comparisons. Moreover, \emph{the constant velocity propagation method is preferred, \emph{i.e.}, we mainly evaluate our vanilla RSSR method} based on the constant velocity model due to its simplicity.
Subsection~\ref{Analysis_DeepUnrollNet} provides additional instructions on DeepUnrollNet \cite{liu2020deep}.
Since the Carla-RS dataset provides the GT occlusion masks, for better evaluation, we conduct quantitative experiments including: the Carla-RS dataset with occlusion mask (\emph{CRM}), the Carla-RS dataset without occlusion mask (\emph{CR}), and the Fastec-RS dataset (\emph{FR}).

\begin{figure}[!t]
	\centering
	\includegraphics[width=0.48\textwidth]{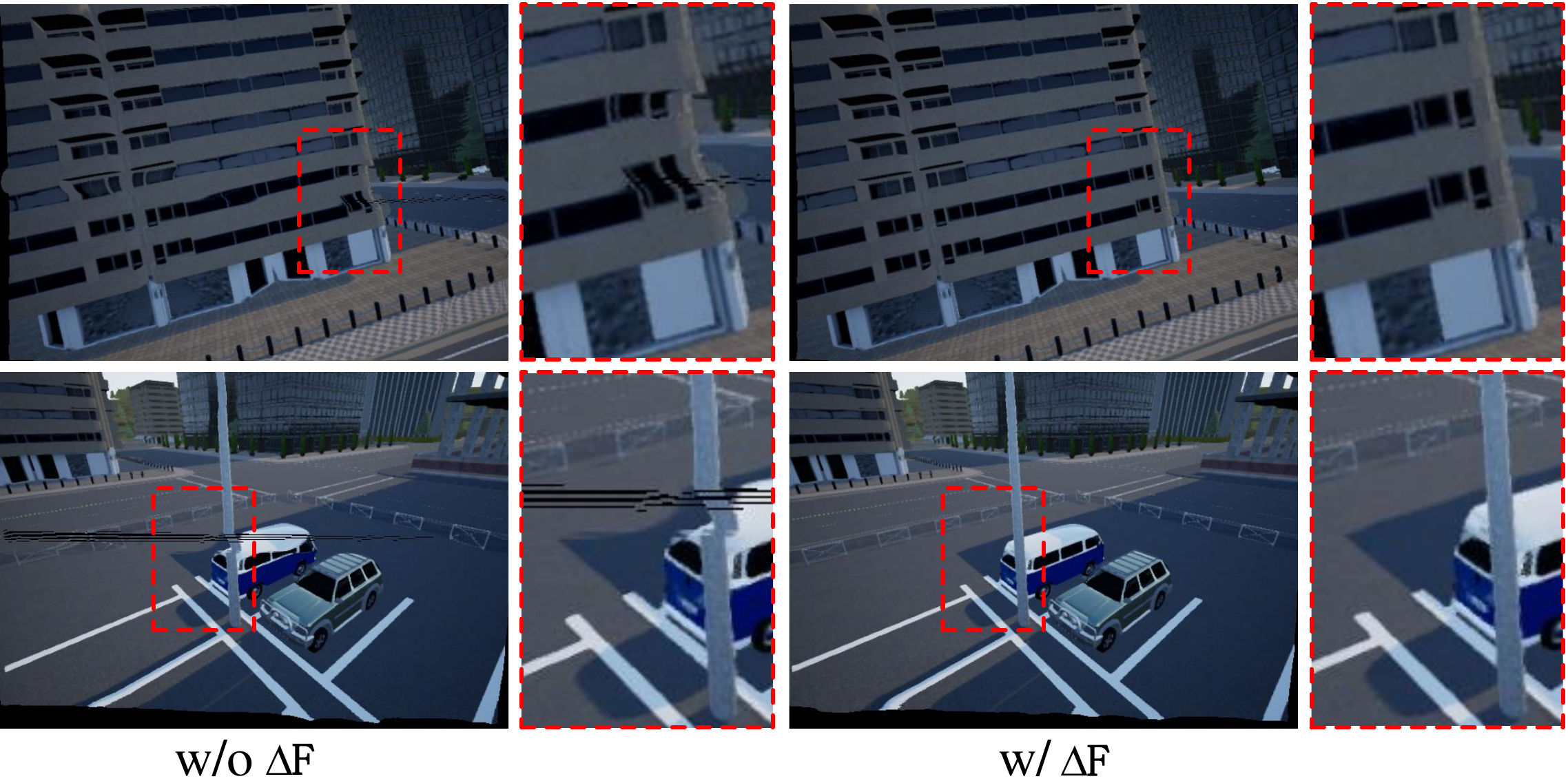}\vspace{-2.0mm}
	\caption{Effectiveness of optical flow residual estimation layer. It can effectively alleviate the artifacts and holes at the boundaries caused by optical flow misalignment, leading to higher quality results. \label{fig:exp_Carla_wo_delta_F}
	}
	\vspace{-1.0mm}
\end{figure}

\begin{table}[!t]
	\footnotesize
	\caption{Ablation study on the selection and training strategy of network $\mathcal{F}$. We employ different optical flow estimation baselines (RAFT \cite{teed2020raft} and PWC-Net \cite{sun2018pwc}), while testing the effect of freezing their parameters during training.}\label{ablation_F}\vspace{-1.5mm}
	\centering
	\begin{tabular}{cccccccc}
		\hline
		\multirow{2}{*}{RAFT} &\multirow{2}{*}{PWC-Net} &\multirow{2}{*}{Freeze} &\multicolumn{2}{c}{PSNR$\uparrow$ (dB)}  & &\multicolumn{2}{c}{SSIM$\uparrow$}      \\ \cline{4-5} \cline{7-8}
		&&&CRM             & FR            & & CR            & FR            \\ \hline
		$\checkmark$ &              & $\checkmark$  & 29.81          & 20.63  &  & \textbf{0.87}  & 0.77  \\
		$\checkmark$ &              & $\times$      & 26.75          & 20.38  &  & 0.84      & 0.71  \\
		& $\checkmark$ & $\checkmark$ & 29.44       & 20.64       &  & 0.86      & 0.77  \\
		& $\checkmark$ & $\times$   & \textbf{30.17} & \textbf{21.26}&  & \textbf{0.87} & \textbf{0.78}  \\ \hline
	\end{tabular}
\end{table}

\subsection{Ablation Studies}
To demonstrate the effectiveness of each component in our proposed network, we evaluate the controlled comparisons over network $\mathcal{F}$, network $\mathcal{U}$, and loss function $\mathcal{L}$, respectively.
We train these variations using the same strategy as aforementioned in Subsection~\ref{Training_strategy}.
We will explain each of them in detail below.

\vspace{0.5mm}
\subsubsection{Ablation on the selection and training strategy of network $\mathcal{F}$}
We first replace PWC-Net \cite{sun2018pwc} with the state-of-the-art (SOTA) optical flow estimation baseline RAFT \cite{teed2020raft}.
Then, we analyze the influence of different training strategies of network $\mathcal{F}$, including parameter freezing and model initialization during training, \emph{i.e.},
\begin{itemize}
	\item \emph{$\mathcal{F}$-Scra}: We initialize $\mathcal{F}$ from scratch and optimize it with the whole model.
	\item \emph{$\mathcal{F}$-Pret}: We initialize $\mathcal{F}$ from the pre-trained model provided by \cite{sun2018pwc} and optimize it with the whole model.
\end{itemize}

The quantitative results are summarized in Table~\ref{ablation_F} and Table~\ref{ablation_U_T}.
We can observe that RAFT contributes slightly worse than PWC-Net to the high framerate GS video extraction in our implementation.
Freezing the network parameters when using RAFT can significantly improve performance, but it has the opposite benefit when combined with PWC-Net.
When initializing from the pre-trained model, especially with the addition of fine-tuning, our method shows a substantial performance improvement. After jointly optimizing the whole network together with PWC-Net, the overall performance is further improved, which is better capable of exploiting the concealed motion between scanlines as well as the scene structure.

\vspace{0.5mm}
\subsubsection{Ablation on the design of network $\mathcal{U}$}
We further investigate the contribution of each component in network $\mathcal{U}$ as follows:
\begin{itemize}
	\item w/o \emph{$\mathbf{T}$}: We remove the normalized scanline offset in Eq.~\eqref{eq:22},
	and replace the \emph{Sigmoid} function with the \emph{Tanh} function in network $\mathcal{U}$ to uniformly map the correlation factor prediction of each pixel to the interval of $(-1,1)$.
	\item w/o P\ref{p2}: We remove the underlying constraint defined in Proposition \ref{p2}, \emph{i.e.}, use the normalized scanline offset and remove the \emph{Sigmoid} function, such that the network $\mathcal{U}$ output the correlation maps $\mathbf{C}_{1 \to m}, \mathbf{C}_{2 \to m} \in \mathbb{R}^{h\times w}$, \emph{i.e.}, learn a generalized middle-scanline undistortion flow.
	\item w/o \emph{$\Delta\mathbf{F}$}: We remove the optical flow residual estimation layer in network $\mathcal{U}$, \emph{i.e.}, $\Delta\mathbf{F}_{1 \to 2}=\Delta\mathbf{F}_{2 \to 1}=\mathbf{0}$ in Eq.~\eqref{eq:23}.
\end{itemize}

We report the results in Table~\ref{ablation_U_T} and Fig.~\ref{fig:exp_Carla_wo_delta_F}. One can see that the explicit constraint of the normalized scanline offset benefits the learning of the scanline-dependent nature of the RS undistortion flow, which is consistent with the observation in \cite{liu2020deep}.
Moreover, imposing the intrinsic geometric constraint in Proposition \ref{p2} promotes better convergence of our vanilla network and decisively improves \emph{CRM} by 4.22 dB.
And adding the optical flow residual estimation layer is effective to facilitate the edge alignment and improve the robustness of the proposed model in the extreme case, thereby recovering more complete and accurate image details.

\vspace{0.5mm}
\subsubsection{Ablation on the loss function}
We show the results of training our models under different loss function settings in Table~\ref{ablation_U_T}. We remove each loss term from the overall loss function $\mathcal{L}$ respectively. Without $\mathcal{L}_w$ indicates freezing the parameters of PWC-Net in Table~\ref{ablation_F}. Forcing the smoothness of the estimated flows has a particularly positive effect on improving the performance.
Our loss function is effective as the performance of adopting all loss terms is the best.

\begin{table}[!t]
	\footnotesize
	\caption{Effectiveness of different components of our vanilla RSSR method on the Carla-RS dataset.}\label{ablation_U_T}
	\vspace{-1.5mm}
	\centering
	\begin{tabular}{lcccccc}
		\hline
		\multirow{2}{*}{} & \multicolumn{2}{c}{PSNR$\uparrow$ (dB)} &  & SSIM$\uparrow$    &  & LPIPS$\downarrow$ \\ \cline{2-3} \cline{5-5} \cline{7-7}
		& CRM        & CR        &  & CR &  & CR \\ \hline
		\emph{$\mathcal{F}$-Scra} & 27.37       & 24.22       &  & 0.80   &  & 0.0804 \\
		\emph{$\mathcal{F}$-Pret} & 29.89       & 24.61       &  & 0.86   &  & 0.0697 \\ \hline\hline
		w/o $\mathbf{T}$        & 25.43       & 22.55       &  & 0.82   &  & 0.1116 \\
		w/o {P}\ref{p2}         & 25.95       & 22.98       &  & 0.78   &  & 0.0812 \\
		w/o $\Delta \mathbf{F}$ & 29.12       & 24.29       &  & 0.85   &  & 0.0725 \\ \hline\hline
		w/o $\mathcal{L}_r$ & 29.44       & 24.35       &  & 0.86   &  & 0.0713 \\
		w/o $\mathcal{L}_p$ & 29.82       & 24.61       &  & 0.86   &  & 0.0706 \\
		w/o $\mathcal{L}_s$ & 29.28       & 24.44       &  & 0.86   &  & 0.0725 \\ \hline\hline
		full model          & \textbf{30.17} & \textbf{24.78} &  & \textbf{0.87} &  & \textbf{0.0695} \\ \hline
	\end{tabular}
\end{table}

\begin{figure*}[!t]
	\centering
	\includegraphics[width=0.98\textwidth]{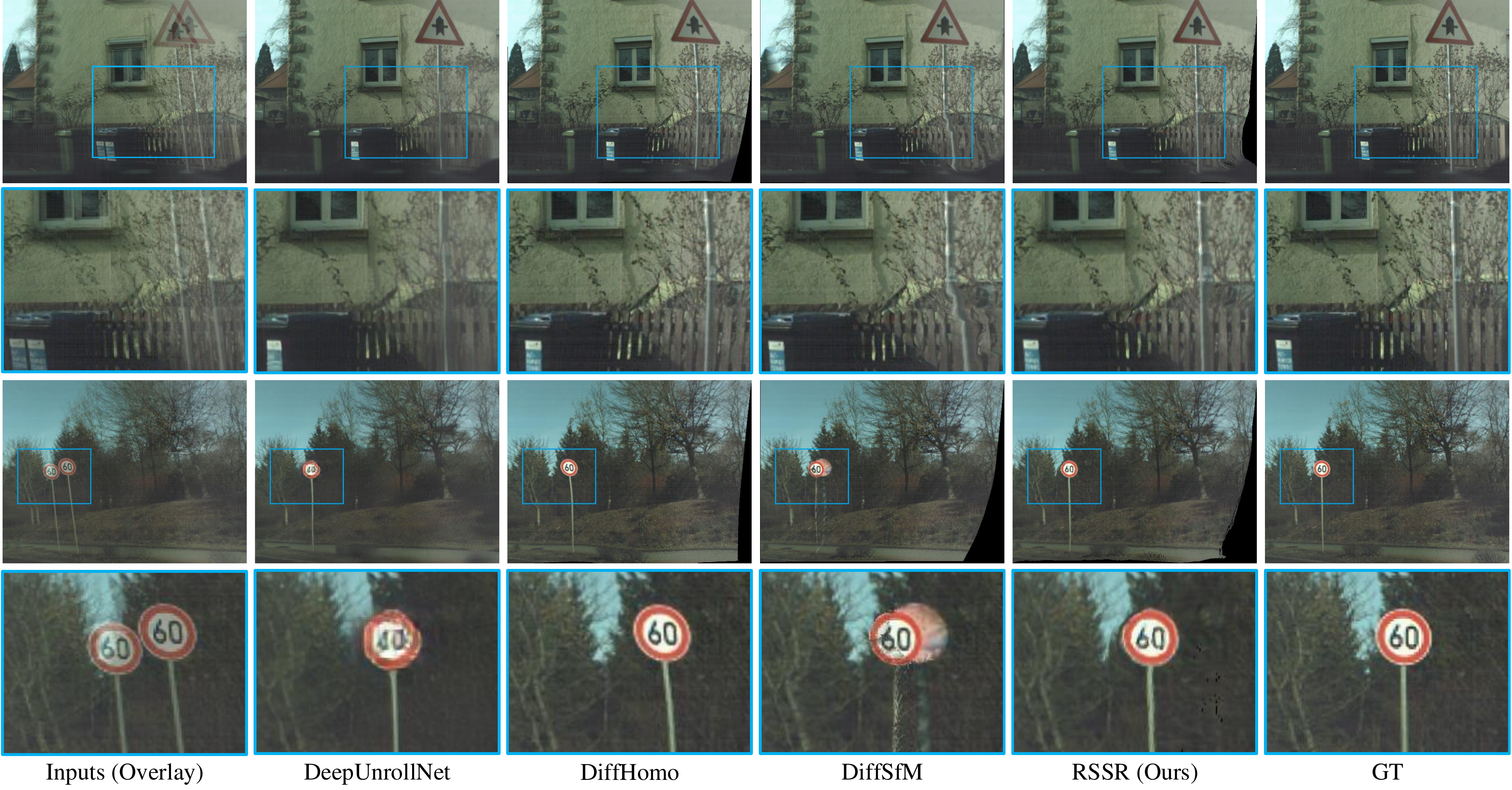}\vspace{-4.0mm}
	\caption{Visual comparisons on the Fastec-RS testing set. We zoom in the correction results according to the blue boxes. Our method fully explores the underlying RS geometry and generates a set of high-quality GS results, in spite of the road sign that is subject to large RS effects.\label{fig:exp_Fastec_zoom}}
	\vspace{-1.0mm}
\end{figure*}

\begin{figure*}[!t]
	\centering
	\includegraphics[width=0.998\textwidth]{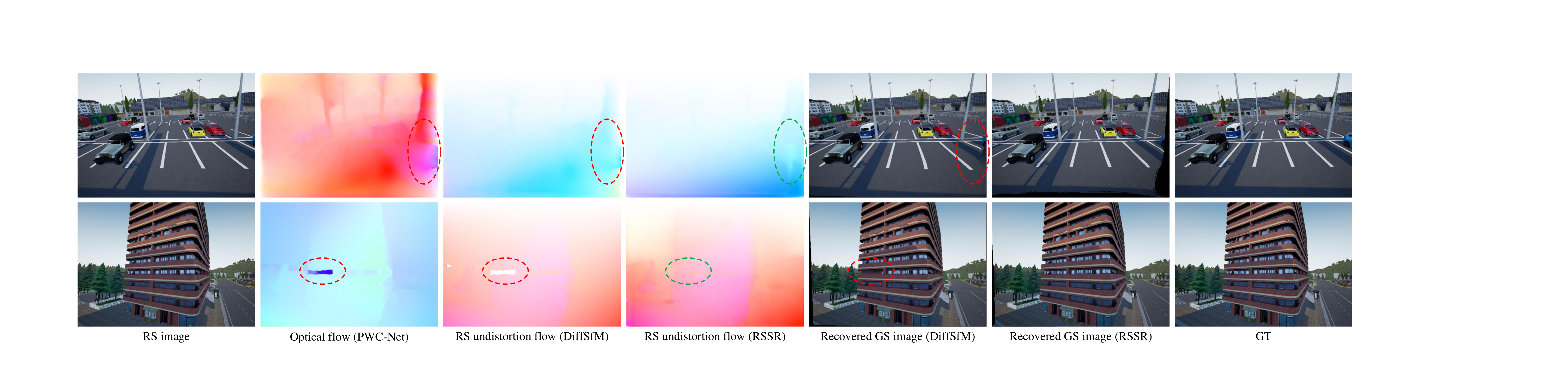}\vspace{-3.2mm}
	\caption{Visual examples of RS undistortion flow obtained by DiffSfM and our method, respectively. Thanks to end-to-end learning, our method can generate more accurate RS undistortion flows (see green circles), and thus synthesize higher quality GS images. \label{fig:VS_Diffsfm_Flows}}
	\vspace{-1.3mm}
\end{figure*}

\begin{table*}[!t]
	\footnotesize
	\caption{Quantitative comparisons on recovering GS images corresponding to the first scanline of the second RS frame. {\color{red}\textbf{Bold}} and {\color{blue}\underline{underlined}} numbers represent the best and second-best performance. Our approach performs favorably against existing methods.}\label{comprehensive_comparison}
	\vspace{-3.0mm}
	\centering
	\setlength{\tabcolsep}{3mm}{
		\begin{tabular}{lccccccccc}
			\hline
			\multirow{2}{*}{Method} &\multicolumn{3}{c}{PSNR$\uparrow$ (dB)}      &   & \multicolumn{2}{c}{SSIM$\uparrow$} &   & \multicolumn{2}{c}{LPIPS$\downarrow$}     \\ \cline{2-4} \cline{6-7} \cline{9-10}
			& CRM            & CR             & FR            & & CR            & FR   & & CR            & FR         \\ \hline
			DeepUnrollNet \cite{liu2020deep}        & \color{blue}\underline{26.90}    & \color{red}\textbf{26.46}     & \color{red}\textbf{26.52}   & & \color{blue}\underline{0.81}  & \color{red}\textbf{0.79}    & & \color{blue}\underline{0.0703}     & \color{red}\textbf{0.1222}      \\
			DiffSfM-\emph{\scriptsize PWCNet} \cite{zhuang2017rolling}  & 19.53    & 18.62     & 18.59   & & 0.69     & 0.63   & & 0.2042     & 0.2416       \\
			DiffSfM-\emph{\scriptsize RAFT} \cite{zhuang2017rolling} & 24.20    & 21.28     & 20.14   & & 0.78     & 0.70  & & 0.1322     & 0.1789        \\
			RSSR (Ours) & \color{red}\textbf{30.17} & \color{blue}\underline{24.78} & \color{blue}\underline{21.26} & & \color{red}\textbf{0.87} & \color{blue}\underline{0.78} & & \color{red}\textbf{0.0695}     & \color{blue}\underline{0.1424}\\ \hline
	\end{tabular}}
	\vspace{-1.0mm}
\end{table*}

\subsection{Comparisons with SOTA RS Correction Methods}\label{Comparison_RSC}
We evaluate the proposed RSSR method against the following RS correction algorithms:
\begin{itemize}
	\item[-] \textbf{DeepUnrollNet} \cite{liu2020deep}: This is the SOTA learning-based two-image RS correction method. It has no ability to produce high framerate GS video sequences. Note that the pre-trained model released by the authors is used to estimate the GS image of the middle scanline of the second RS frame (See Subsection~\ref{Analysis_DeepUnrollNet} for more details). For a fair comparison, we retrain it to synthesize the GS image corresponding to the first scanline of the second RS frame.
	\item[-] \textbf{DiffSfM} \cite{zhuang2017rolling} and \textbf{DiffHomo} \cite{zhuang2020homography}: They are traditional RS correction methods. We implement two versions of DiffSfM using optical flow inputs obtained by PWC-Net and RAFT, respectively.
	Since DiffHomo is not open source, we asked the authors to test some examples. Note that these two methods are suitable for recovering the GS image at the first scanline of the first RS frame. To ensure a fair comparison, for the input of three consecutive RS images, we run our method using the first two RS frames as input, and perform DiffSfM and DiffHomo using the last two RS images as input. In this way, we can generate and align to the GS image corresponding to the first scanline of the second RS frame.
\end{itemize}

We report the quantitative results in Table~\ref{comprehensive_comparison}. The better optical flow obtained by RAFT can improve the estimation accuracy of RS geometry, thus promoting the performance of RS correction in DiffSfM, which also indicates the importance of mining the underlying RS geometry.
Our model performs favorably against all the compared methods in the Carla-RS dataset and competes on par with DeepUnrollNet in the Fastec-RS dataset, since the Carla-RS dataset is more consistent with the constant motion assumption. Note that, compared with DeepUnrollNet, the black holes in our corrected GS images will reduce the PSNR score of our method without using masks. Very importantly, our method can generate high framerate and visually pleasing GS video sequences.
We provide the visual correction results with noticeable RS distortions in Fig.~\ref{fig:exp_Fastec_zoom}, where DiffHomo and DiffSfM fail to correct the geometric RS distortions and DeepUnrollNet causes loss of local image details.
Compared with the off-the-shelf RS correction algorithms, our pipeline effectively restores higher-quality GS images.

Furthermore, considering the existence of an analytic solution for the correction map ${\mathbf C}$ in Eq.~\eqref{eq:16} (which can be estimated by DiffSfM \cite{zhuang2017rolling}), we visualize and compare the intermediate process.
Since RS undistortion flow is essentially the core of GS image recovery, we illustrate it estimated by our method and DiffSfM in Fig.~\ref{fig:VS_Diffsfm_Flows}. In addition to the complex computation, DiffSfM is powerless to correct the initial incorrect optical flow and thus produces local artifacts, as shown by the red circles. Furthermore, when DiffSfM estimates erroneous RS geometry, such as biased camera motion, the RS effect will not be truly removed. In contrast, since our approach supports end-to-end learning, the inaccuracy of the initial optical flow can be automatically adjusted to yield a more accurate task-oriented RS undistortion flow, thereby recovering a more reliable GS image.

\begin{figure*}[!t]
	\centering
	\includegraphics[width=0.88\textwidth]{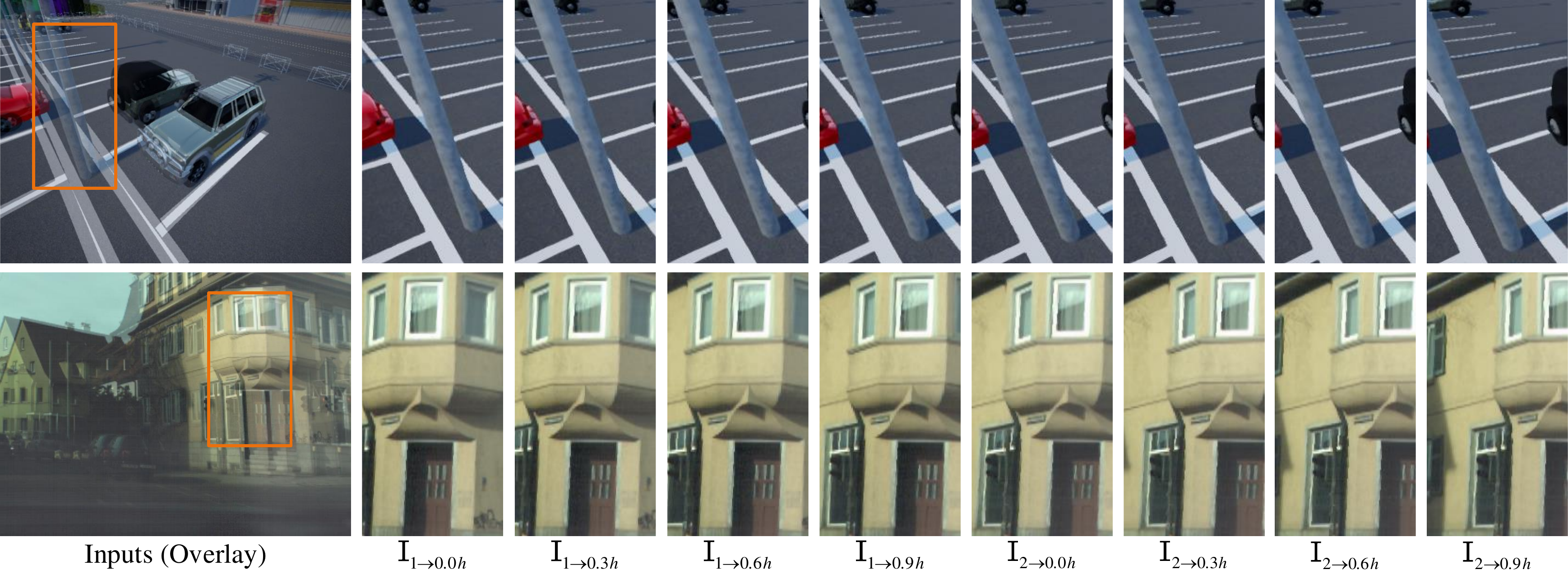}\vspace{-3.7mm}
	\caption{Examples of GS video frame interpolation by our vanilla RSSR method. We zoom in the intermediate results according to the orange boxes. For example, ``${\mathbf I}_{1 \to 0.3h}$'' denotes the corrected global-shutter image corresponding to $0.3h$-th scanline of RS frame 1. Our method not only preserves the temporal smoothness but also corrects the rolling-shutter artifacts. \label{fig:multiple_frames}}
	\vspace{-1.0mm}
\end{figure*}

\begin{figure*}[!t]
	\centering
	\includegraphics[width=1.0\textwidth]{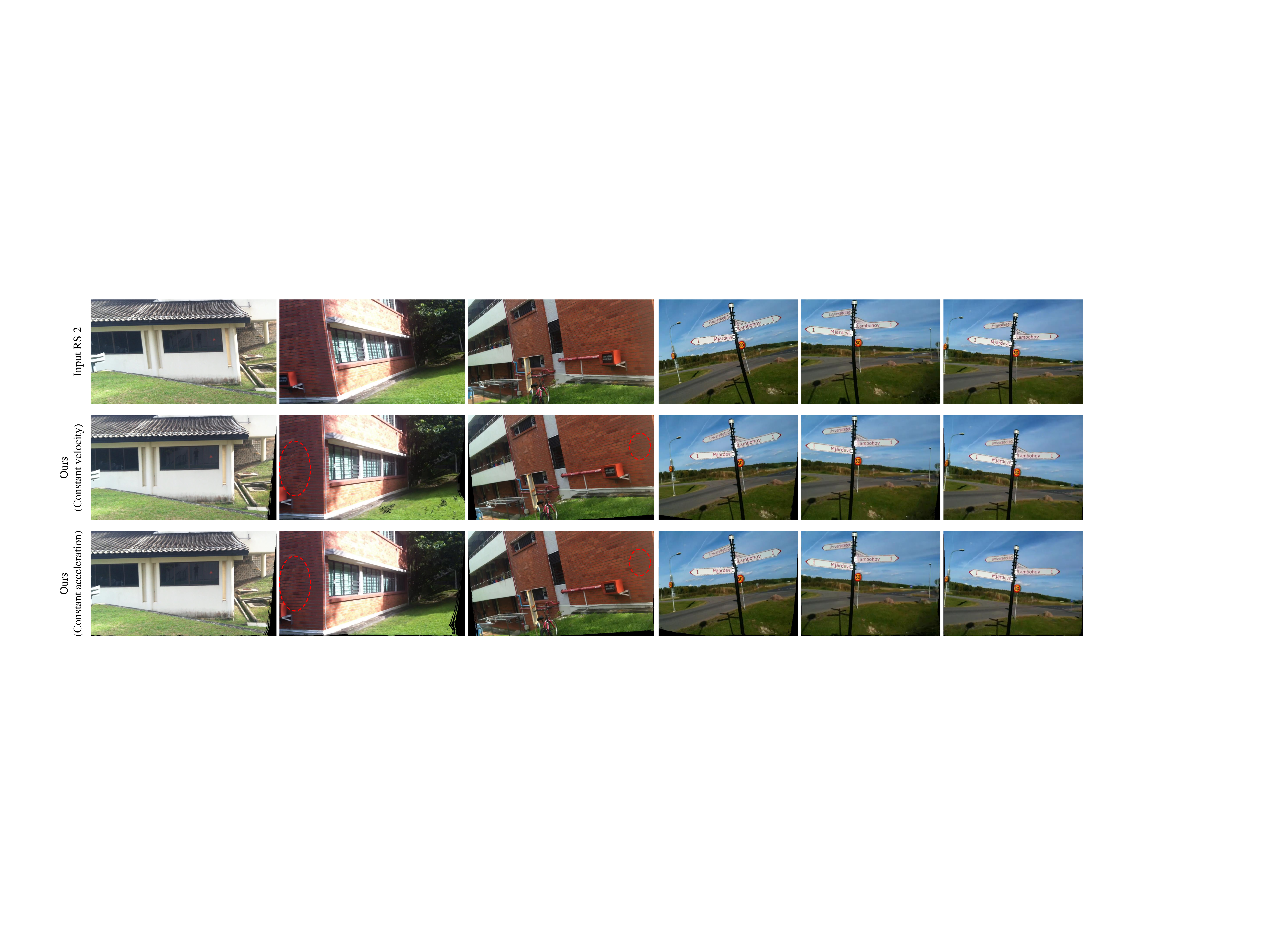}\vspace{-3.5mm}
	\caption{Qualitative image correction results on real data with obvious RS distortion provided by \cite{zhuang2017rolling} and \cite{forssen2010rectifying}, respectively. The left three columns are from \cite{zhuang2017rolling} and the right three columns are from \cite{forssen2010rectifying}. Our pipeline can effectively remove RS distortion as a whole. In particular, our more general constant acceleration model produces more trustworthy images on unseen data, as indicated by the red circles. \label{fig:generalization_vel_acc}
	\vspace{-1.5mm}
	}
\end{figure*}


\begin{figure}[!t]
	\centering
	\includegraphics[width=0.478\textwidth]{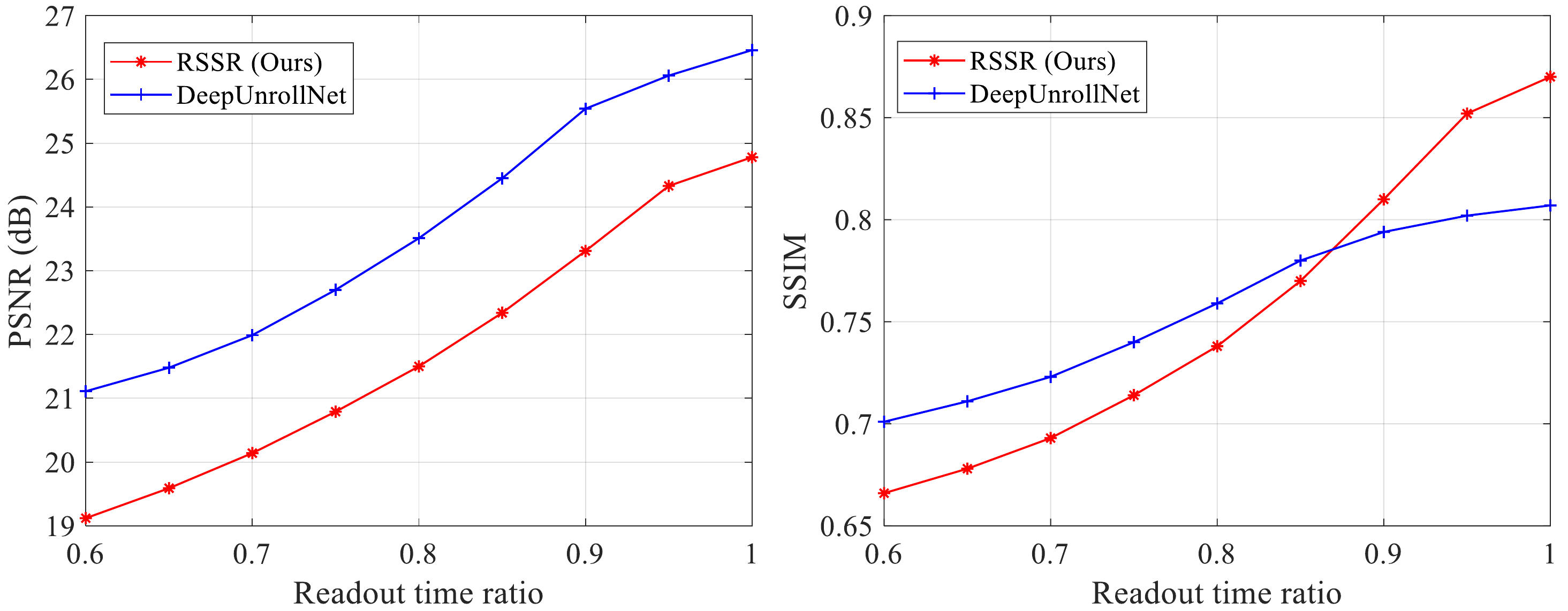}\vspace{-3.0mm}
	\caption{Robustness evaluation of the readout time ratio. (PSNR \emph{vs.} Readout time ratio and SSIM \emph{vs.} Readout time ratio)\label{fig:varying_rtr}
	}
	\vspace{-3.0mm}
\end{figure}

\begin{table*}[!t]
	\footnotesize
	\caption{Quantitative comparisons on recovering GS images corresponding to the first scanline of the second RS frame under the constant velocity propagation and the constant acceleration propagation, respectively.} \label{comparison_RSSR_plus_plus}
	\vspace{-3.0mm}
	\centering
	\setlength{\tabcolsep}{3mm}{
		\begin{tabular}{lccccccccc}
			\hline
			\multirow{2}{*}{Our Model} &\multicolumn{3}{c}{PSNR$\uparrow$ (dB)}      &   & \multicolumn{2}{c}{SSIM$\uparrow$} &   & \multicolumn{2}{c}{LPIPS$\downarrow$}     \\ \cline{2-4} \cline{6-7} \cline{9-10}
			& CRM            & CR             & FR            & & CR            & FR   & & CR            & FR         \\ \hline
			RSSR (\emph{\scriptsize Constant velocity}) & {30.17} & {24.78} & {21.26} & & {0.867} & {0.784} & & {0.0695}     & {0.1424} \\
			RSSR (\emph{\scriptsize Constant acceleration}) & \textbf{30.35} & \textbf{24.83} & \textbf{21.43} & & \textbf{0.870} & \textbf{0.790} & & \textbf{0.0688}     & \textbf{0.1379}       \\ \hline
	\end{tabular}}
	\vspace{-1.5mm}
\end{table*}

\subsection{Generating Multiple GS Video Frames}
We generate multiple GS video frames corresponding to different scanlines, as shown in Fig.~\ref{fig:multiple_frames}.
Moreover, we attach a \emph{supplementary video} to dynamically show the reversed GS video clips, involving a cropping operation.
In principle, we can produce videos with arbitrary frame rates. Our RSSR network learns to solve the complex RS geometry embedded in the consecutive RS frames, so it can robustly and accurately recover photorealistic time-continuous GS images.
Overall, our method not only has the advantage of RS correction at a specific scanline time, but also has the superior ability to restore GS images at any scanline time.

\begin{figure*}[!t]
	\centering
	\setlength{\tabcolsep}{0.025cm}
	\setlength{\itemwidth}{3.68cm}
	\hspace*{-\tabcolsep}\begin{tabular}{cccc}
		{\small Input RS 2} & {\footnotesize Ours (Constant velocity)} & {\small Ours (Constant acceleration)} & {\small Ground truth} \\
		\includegraphics[width=\itemwidth]{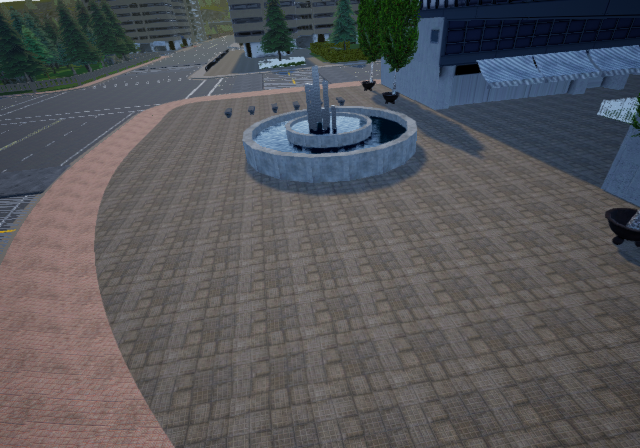}
		&
		\includegraphics[width=\itemwidth]{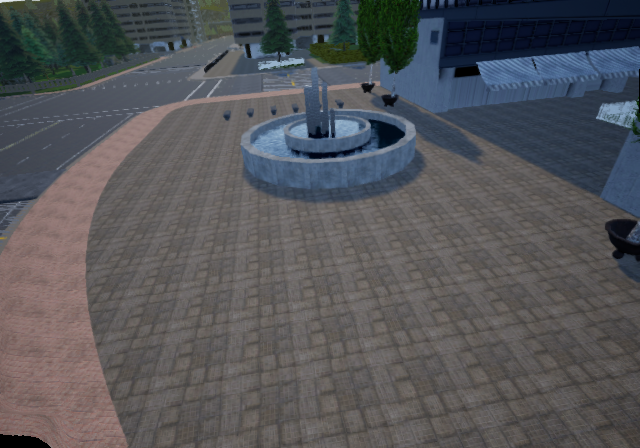}
		&
		\includegraphics[width=\itemwidth]{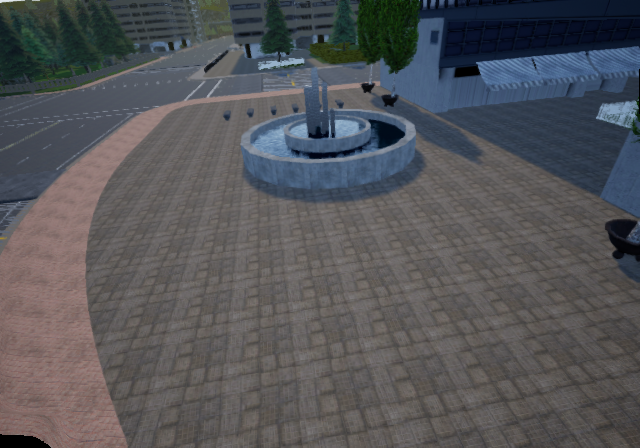}
		&
		\includegraphics[width=\itemwidth]{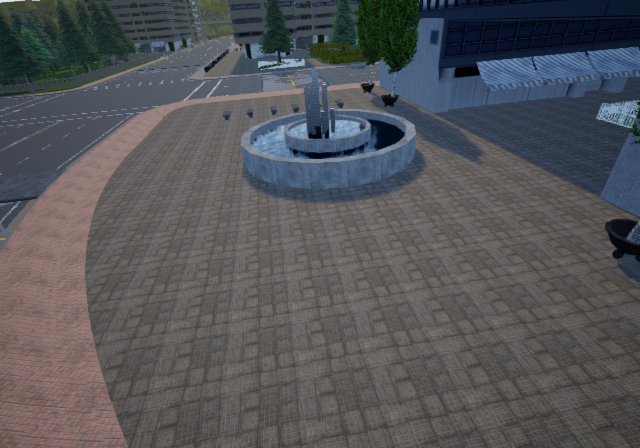}
		\vspace{-0.5mm} \\
		\footnotesize (23.95/0.480/0.0927) & \footnotesize (30.47/0.884/0.0453) & \footnotesize \textbf{(30.55/0.891/0.0449)} & \footnotesize (inf./1.000/0.0000)
		\\
		\includegraphics[width=\itemwidth]{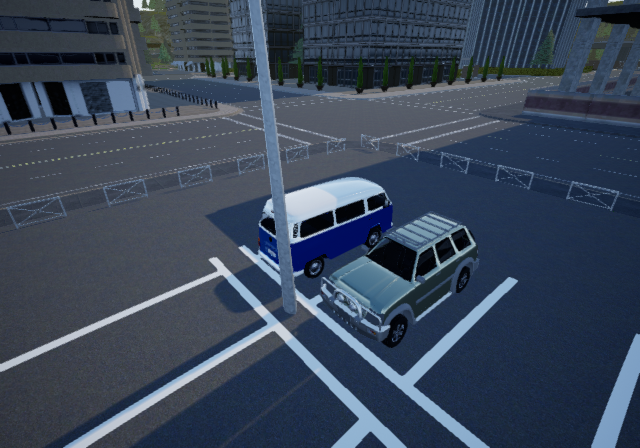}
		&
		\includegraphics[width=\itemwidth]{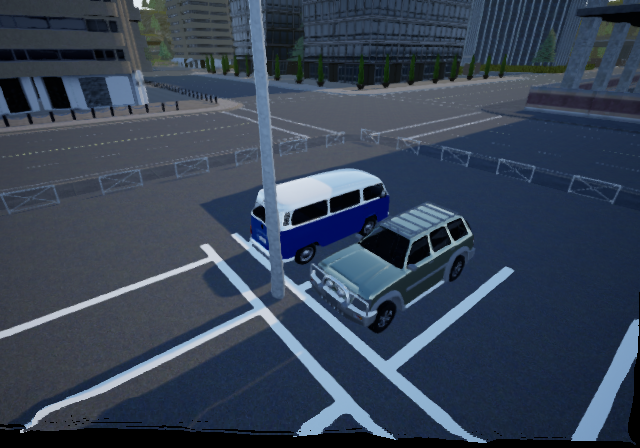}
		&
		\includegraphics[width=\itemwidth]{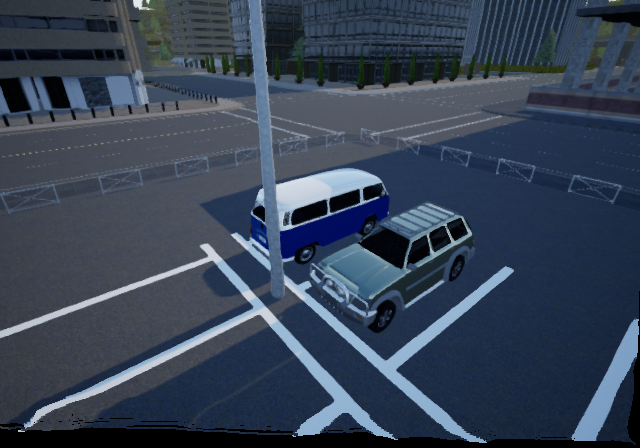}
		&
		\includegraphics[width=\itemwidth]{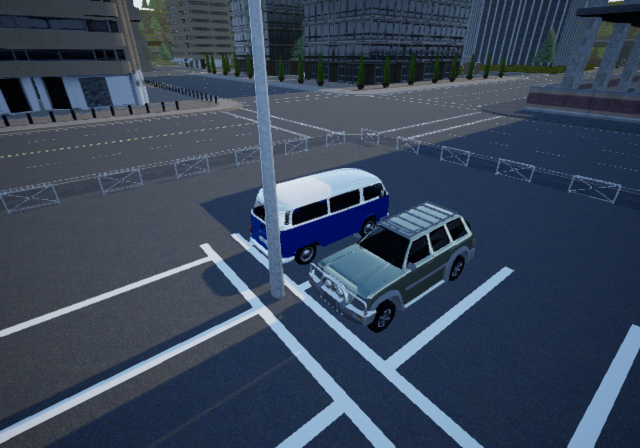}
		\vspace{-0.5mm} \\
		\footnotesize (15.50/0.487/0.2317) & \footnotesize (22.59/0.815/0.0823) & \footnotesize \textbf{(22.87/0.833/0.0817)} & \footnotesize (inf./1.000/0.0000)
		\\
		\includegraphics[width=\itemwidth]{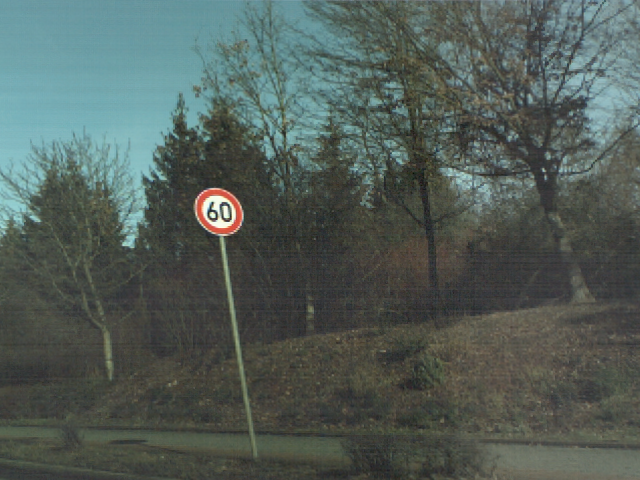}
		&
		\includegraphics[width=\itemwidth]{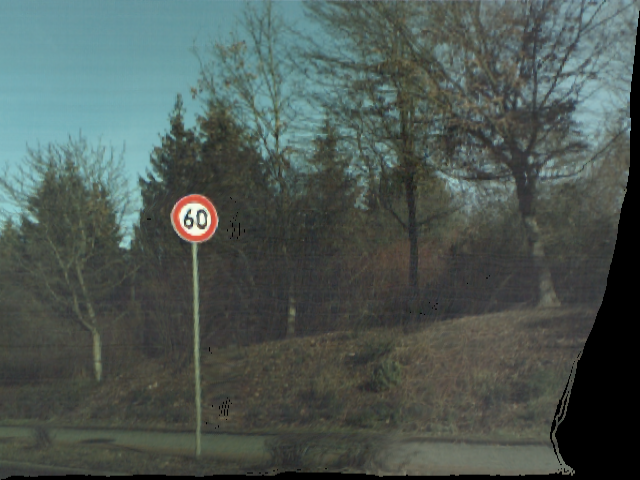}
		&
		\includegraphics[width=\itemwidth]{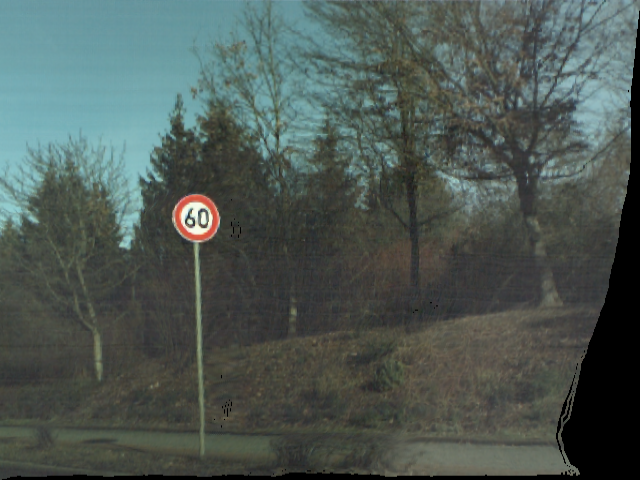}
		&
		\includegraphics[width=\itemwidth]{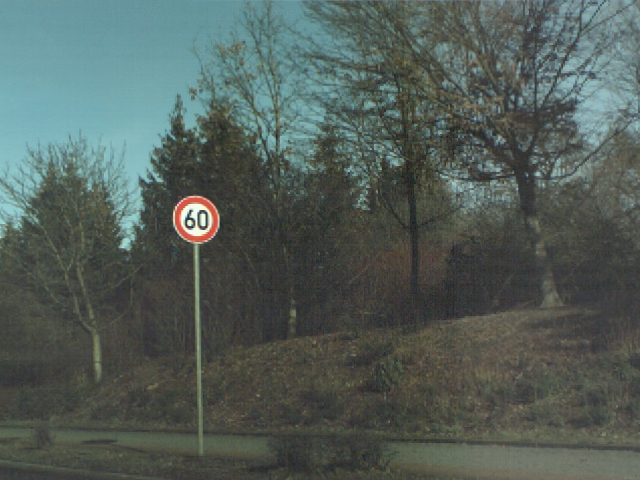}
		\vspace{-0.5mm} \\
		\footnotesize (21.55/0.515/0.1991) & \footnotesize (20.78/0.664/0.1983) & \footnotesize \textbf{(21.12/0.683/0.1913)} & \footnotesize (inf./1.000/0.0000)
		\\
		\includegraphics[width=\itemwidth]{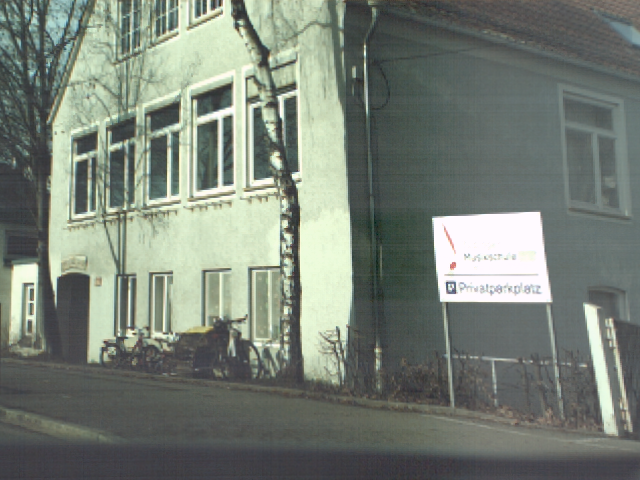}
		&
		\includegraphics[width=\itemwidth]{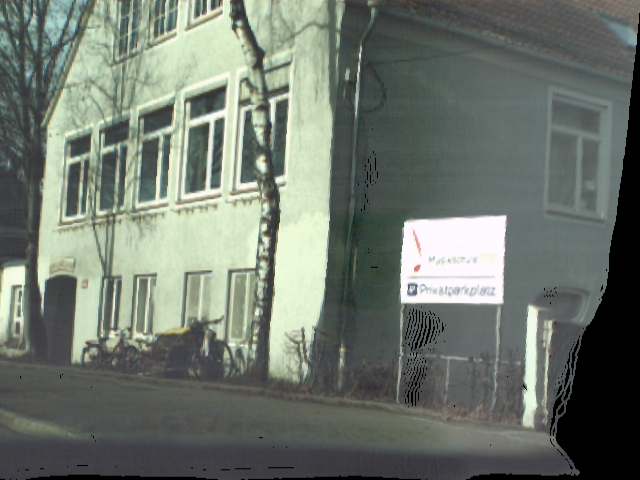}
		&
		\includegraphics[width=\itemwidth]{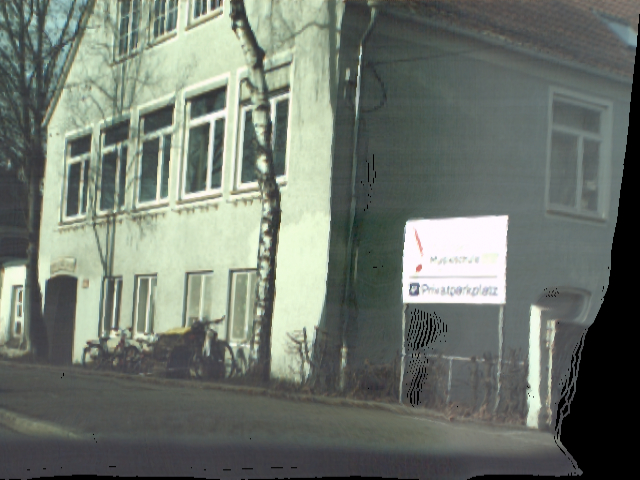}
		&
		\includegraphics[width=\itemwidth]{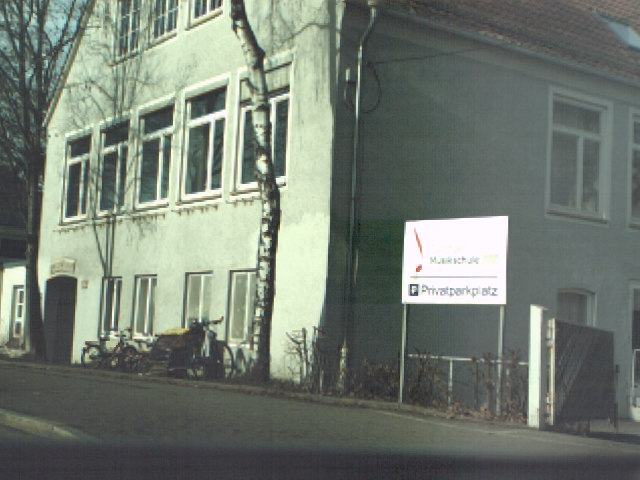}
		\vspace{-0.5mm} \\
		\footnotesize (13.93/0.516/0.2545) & \footnotesize (19.34/0.749/0.1977) & \footnotesize \textbf{(20.16/0.787/0.1796)} & \footnotesize (inf./1.000/0.0000) 	
	\end{tabular}\vspace{-0.2cm}
	\caption{Metrics performance (\emph{i.e.}, PSNR$\uparrow$/SSIM$\uparrow$/LPIPS$\downarrow$) on estimating the global-shutter image of the first scanline under the constant velocity model and the constant acceleration model, respectively. The overall quality of the recovered global-shutter image is further upgraded by applying a more general constant acceleration propagation.}
	\label{fig:vel_vs_acc}
	\vspace{-1.0mm}
\end{figure*}

\begin{figure}[!t]
	\centering
	\begin{minipage}[t]{0.494\linewidth}
		\centering
		\includegraphics[width=1\linewidth]{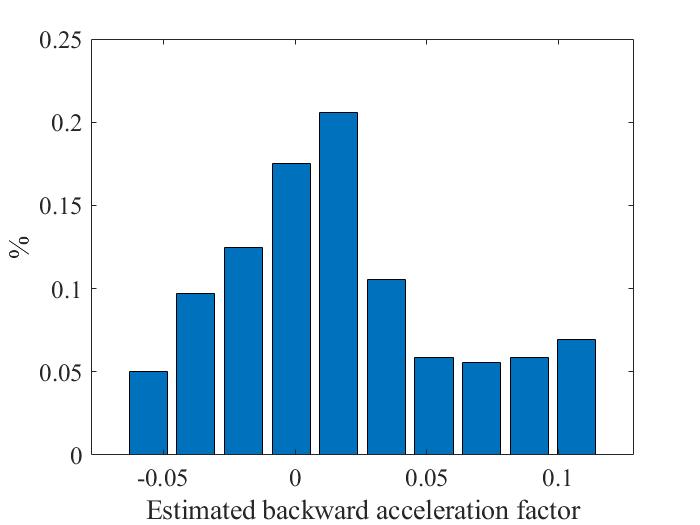}\\
		{\footnotesize (a) Carla-RS dataset}
	\end{minipage}
	\begin{minipage}[t]{0.494\linewidth}
		\centering
		\includegraphics[width=1\linewidth]{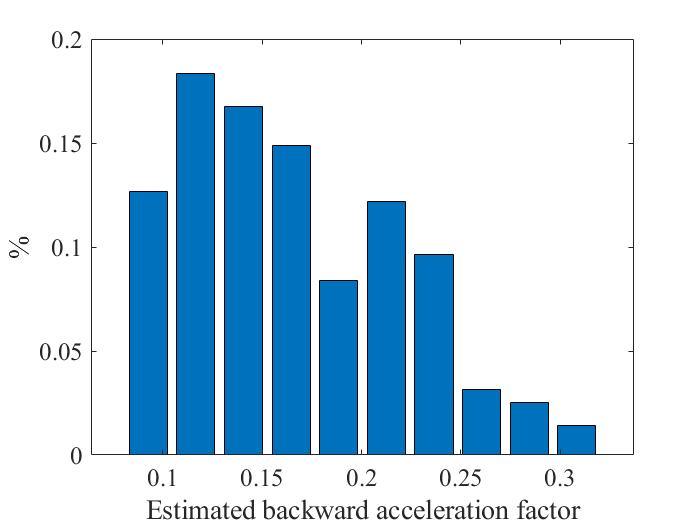}\\
		{\footnotesize (b) Fastec-RS dataset}
	\end{minipage}
	\caption{Frequency distribution histograms of the estimated backward acceleration factors by our Acceleration-Net, where the Carla-RS dataset exhibits much smaller acceleration factor values compared to the subtle intra-frame motion, which is in line with the expectation. \label{fig:backward_acceleration_plot}}
	\vspace{-1.0mm}
\end{figure}

\subsection{Generalization and Robustness Evaluation}
Our learning-based model is trained on the Carla-RS dataset, in which the RS artifacts are mainly caused by uniform camera motion. We apply our method to real RS images provided by \cite{forssen2010rectifying} and \cite{zhuang2017rolling}. The example results are shown in Fig.~\ref{fig:generalization_vel_acc}.
The attached \emph{video} provides more results. Note that the real data was recorded in \cite{forssen2010rectifying} by a rotation-dominated hand-held RS camera and \cite{zhuang2017rolling} collected 720p RS images at 30 fps.
The results reveal that our method owns good generalization ability and can recover visually compelling GS images, due to the learned RS geometry.

Since the readout time ratio $\gamma$ is set to 1 in the Carla-RS and Fastec-RS datasets, the models obtained based on them may be biased towards such training data. Therefore, to verify the robustness to $\gamma$, we use the Carla-RS dataset to simulate continuous RS image pairs with $\gamma$ ranging from 0.6 to 1. For example, to obtain the test data with $\gamma=0.8$, we first discard the pixels in the last $20\%$ scanlines of each RS image and then bilinearly interpolate to the original resolution. We report the evaluation results in Fig.~\ref{fig:varying_rtr}. It can be seen that the RS correction performance gradually decreases as $\gamma$ deviates from the training dataset of $\gamma=1$, due to the fact that our network $\mathcal{U}$ essentially encapsulates the underlying RS geometry with $\gamma=1$
(\emph{i.e.}, the correction map defined in Eq.~\eqref{eq:16}). Actually, $\gamma$ is crucial for two-view RS geometry estimation and needs to be calibrated in advance, which has also been verified in the case of depth estimation \cite{im2018accurate} and RS correction \cite{zhuang2017rolling,zhuang2020homography,vasu2018occlusion}. Our method allows for a slight perturbation of $\gamma$ and thus generalizes well to the real RS data provided by \cite{zhuang2017rolling} and \cite{forssen2010rectifying}, where $\gamma$ is equal to 0.96 and 0.92, respectively. Note that the model exported by DeepUnrollNet is also quite dependent on the training dataset. Thus, developing a method that is robust to readout time ratios will be a future research direction.

\subsection{An Alternative: Constant Acceleration Propagation}
Although our vanilla RSSR method has achieved impressive results in generating high framerate GS video frames under the constant velocity model as mentioned before, an alternative propagation scheme based on Acceleration-Net and Eq.~\eqref{eq:18} under the constant acceleration model achieves tangible improvements over the constant velocity propagation method, as shown in Table~\ref{comparison_RSSR_plus_plus} and Figs.~\ref{fig:vel_vs_acc} and~\ref{fig:generalization_vel_acc}.
In terms of how the RS correction dataset itself is generated in \cite{liu2020deep}, the Carla-RS dataset follows the constant velocity motion model more closely, on which the improvement of performance metrics is therefore not significant. This is also verified in Fig.~\ref{fig:backward_acceleration_plot}. The introduction of the constant acceleration model on the Fastec-RS dataset to extend to the first scanline (or any scanline) is more successful in compensating for the RS effect, as would be expected given more realistic motion assumptions.
The higher quality first-scanline GS image recovery is illustrated in Figs.~\ref{fig:vel_vs_acc} and~\ref{fig:generalization_vel_acc}, which demonstrates that the constant acceleration propagation effectively improves the performance and makes the inversion results more faithful to the latent global-shutter images.

\begin{figure*}[!t]
	\centering
	\includegraphics[width=0.98\textwidth]{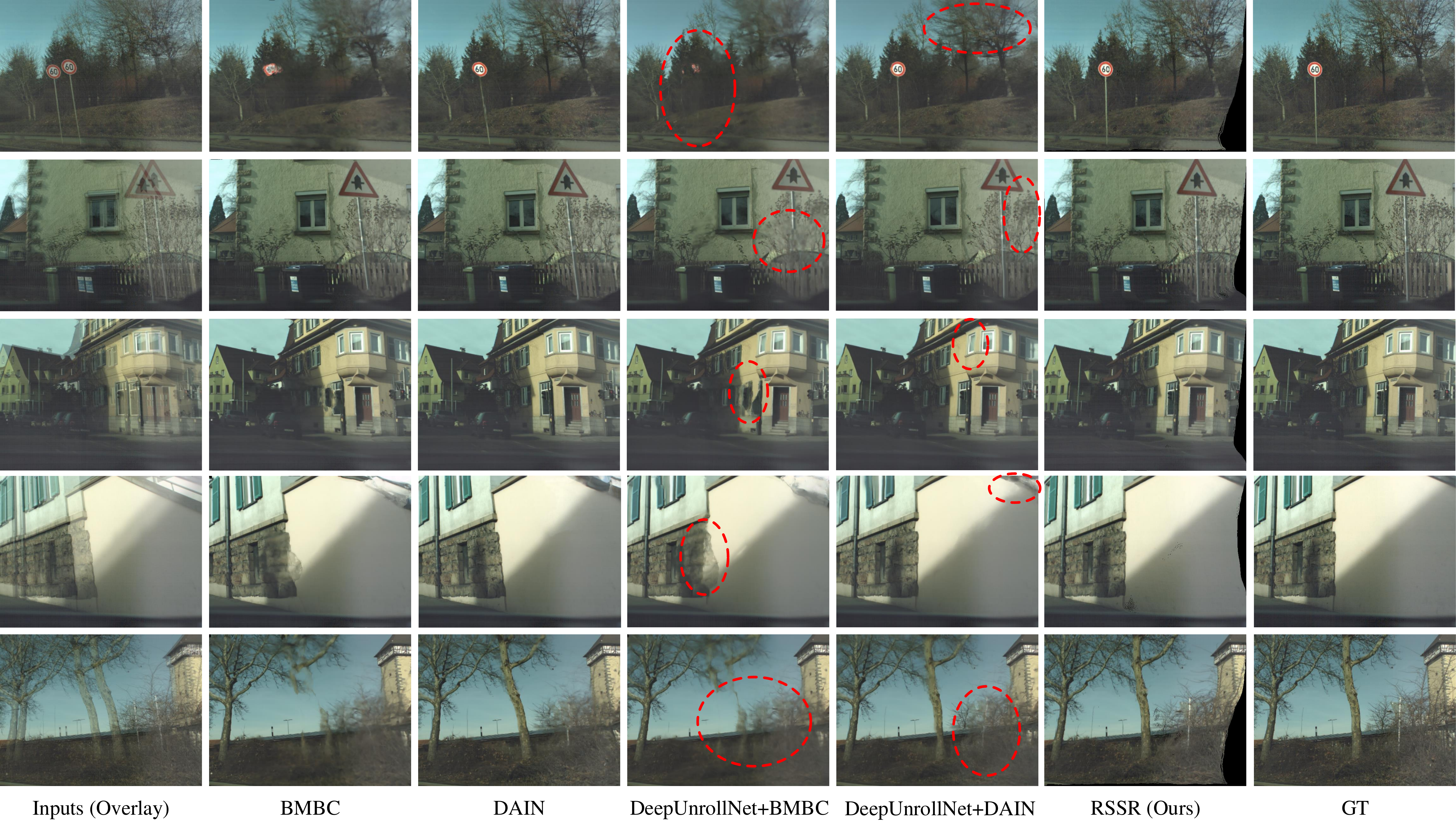}\vspace{-3.0mm}
	\caption{Comparison with VFI methods (BMBC \cite{park2020bmbc} and DAIN \cite{bao2019depth}) and two-stage methods (DeepUnrollNet + BMBC and DeepUnrollNet + DAIN). \label{fig:VS_VFI_etc}
	}
	\vspace{-1.0mm}
\end{figure*}

\begin{table}[!t]
	\footnotesize
	\caption{Superiority of our method over varying two-stage methods on the Carla-RS dataset. We report the quantitative results corresponding to the first and middle scanlines of the RS frame, respectively.}\label{vs_two_stage}
	\centering
	\begin{tabular}{lccc}
		\hline
		{Method} & {PSNR$\uparrow$ (dB)} & SSIM$\uparrow$    & LPIPS$\downarrow$ \\ \cline{2-2} \cline{3-3} \cline{4-4}\hline\hline
		\multicolumn{4}{c}{Corresponding to the first scanline} \\ \hline
		DeepUnrollNet + BMBC  & 27.29       & 0.83     & 0.0980 \\
		DeepUnrollNet + DAIN  & 27.48       & 0.87     & 0.0821 \\
		RSSR (Ours)           & \textbf{30.17} & \textbf{0.87} & \textbf{0.0695} \\ \hline\hline
		\multicolumn{4}{c}{Corresponding to the middle scanline} \\ \hline
		SUNet + BMBC        & 28.51         & 0.85   & 0.1033 \\
		SUNet + DAIN        & 28.63         & 0.85   & 0.0919 \\
		RSSR (Ours)         & \textbf{29.36} & \textbf{0.90} & \textbf{0.0553} \\ \hline
	\end{tabular}
	\vspace{-1.0mm}
\end{table}

\subsection{Versus SOTA Video Frame Interpolation Methods}
The current video frame interpolation algorithms, \emph{e.g.}, BMBC \cite{park2020bmbc} and DAIN \cite{bao2019depth}, have a common implicit assumption that the camera employs a global-shutter mechanism,
where the pixel displacement is controllable and located in the corresponding optical flow. Specifically, linearly scaling the optical flow between 0 and 1 to approximate the required intermediate pixel displacement in order to warp input images.
In contrast, to correct the RS image, as shown in Eq.~\eqref{eq:15}, the pixel displacement is neither a linear function of scanline time (including complex RS geometry) nor within the corresponding optical flow (\emph{i.e.}, the length of the RS undistortion flow may be larger than that of the optical flow, or its direction may be opposite to the optical flow), involving intrinsic non-local operations.
Therefore, because of inherent flaws in the network architectures, the existing video frame interpolation algorithms are incapable of eliminating the RS effect effectively.
We validate this argument in Fig.~\ref{fig:VS_VFI_etc}, which also highlights the superiority of our method in recovering high-quality distortion-free GS video frames.

Furthermore, a naive approach to the RS inversion task at hand would be to cascade RS correction and interpolation methods.
We thus conduct experiments to compare with this so-called two-stage approach.
Since DeepUnrollNet \cite{liu2020deep} and SUNet \cite{fan2021sunet} are fundamentally designed to recover the GS image corresponding to the middle and first scanlines of the second RS frame, respectively, we implement four two-stage methods including four cascades as follows:

\begin{itemize}
	\item[-] \textbf{DeepUnrollNet + BMBC} and \textbf{DeepUnrollNet + DAIN}: Given three consecutive RS images, we first obtain two middle-scanline GS images in sequence using DeepUnrollNet, and then interpolate the GS image corresponding to the first scanline of the third RS image using BMBC and DAIN, respectively.
	\item[-] \textbf{SUNet + BMBC} and \textbf{SUNet + DAIN}: Given three consecutive RS images, we first obtain two first-scanline GS images in sequence using SUNet, and then interpolate the GS image corresponding to the middle scanline of the second RS image using BMBC and DAIN, respectively. For consistent comparison, we do not visualize them.
\end{itemize}

As outlined in Fig.~\ref{fig:VS_VFI_etc}, the two-stage approach is prone to error accumulation and therefore suffers from serious blurring artifacts and local inaccuracies. Predictably, further recursive interpolation of other time steps will exacerbate these image degradations. The quantitative results on the Carla-RS dataset are shown in Table~\ref{vs_two_stage}.
Moreover, it is computationally inefficient.
For instance, to recover 960 GS video frames, the two-stage approach takes about 5 minutes (at least 0.3 seconds to generate a frame with DAIN), while our vanilla RSSR method takes a total of 1.8 seconds.
Since our method simultaneously solves the problems of RS correction and temporal super-resolution in an end-to-end manner, it achieves remarkable performance in terms of both quality and efficiency.



\begin{figure*}[!t]
	\centering
	\includegraphics[width=1.0\textwidth]{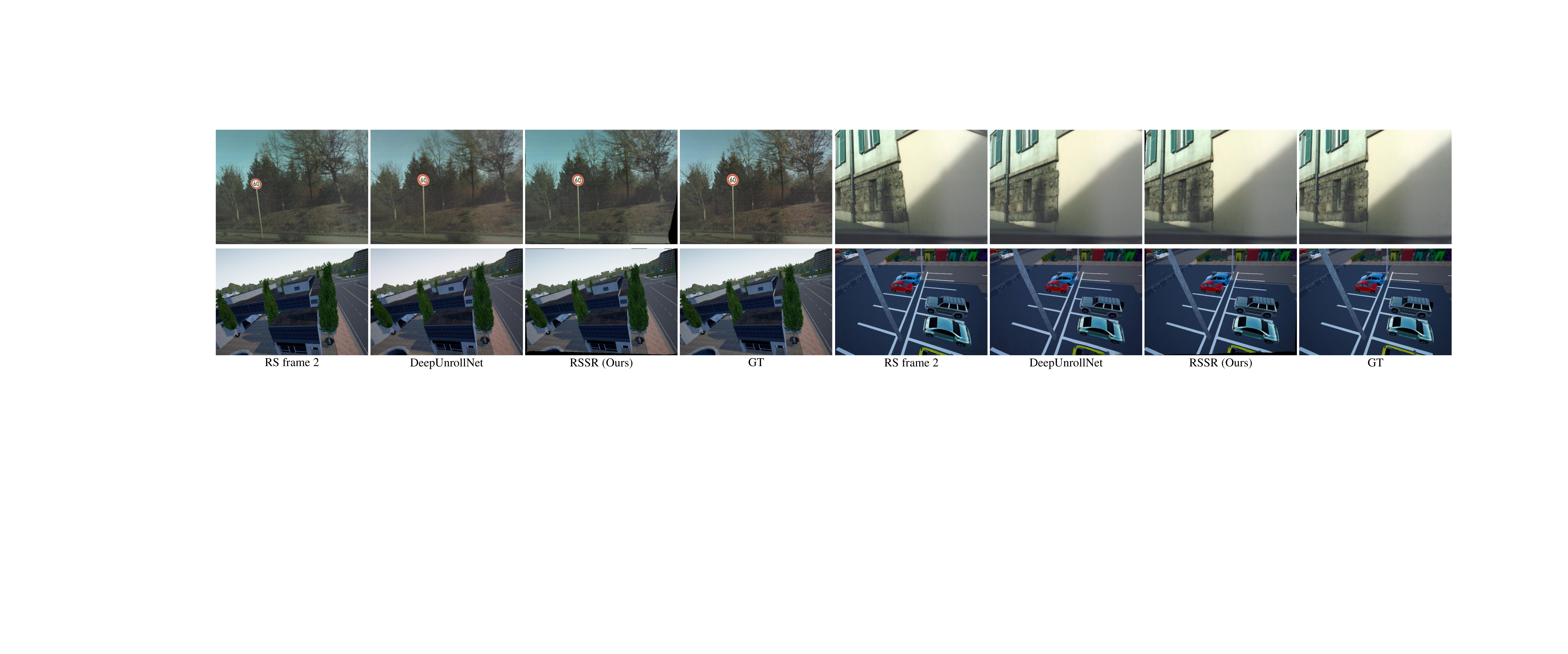}
	\vspace{-8.0mm}
	\caption{Visual comparison of GS images corrected to the middle scanline of the second RS frame. At this time, the corresponding plausible global-shutter image can be reconstructed by DeepUnrollNet \cite{liu2020deep} and our method. Since we have not developed a modulated image decoder as in DeepUnrollNet, our method cannot yet fill the occluded regions. Our RSSR method, however, is able to recover global-shutter video images at any scanline, which is far beyond the reach of DeepUnrollNet. \label{fig:analysy_DeepUnrollNet}}
	\vspace{-4.0mm}
\end{figure*}

\begin{figure}[!t]
	\centering
	\includegraphics[width=0.485\textwidth]{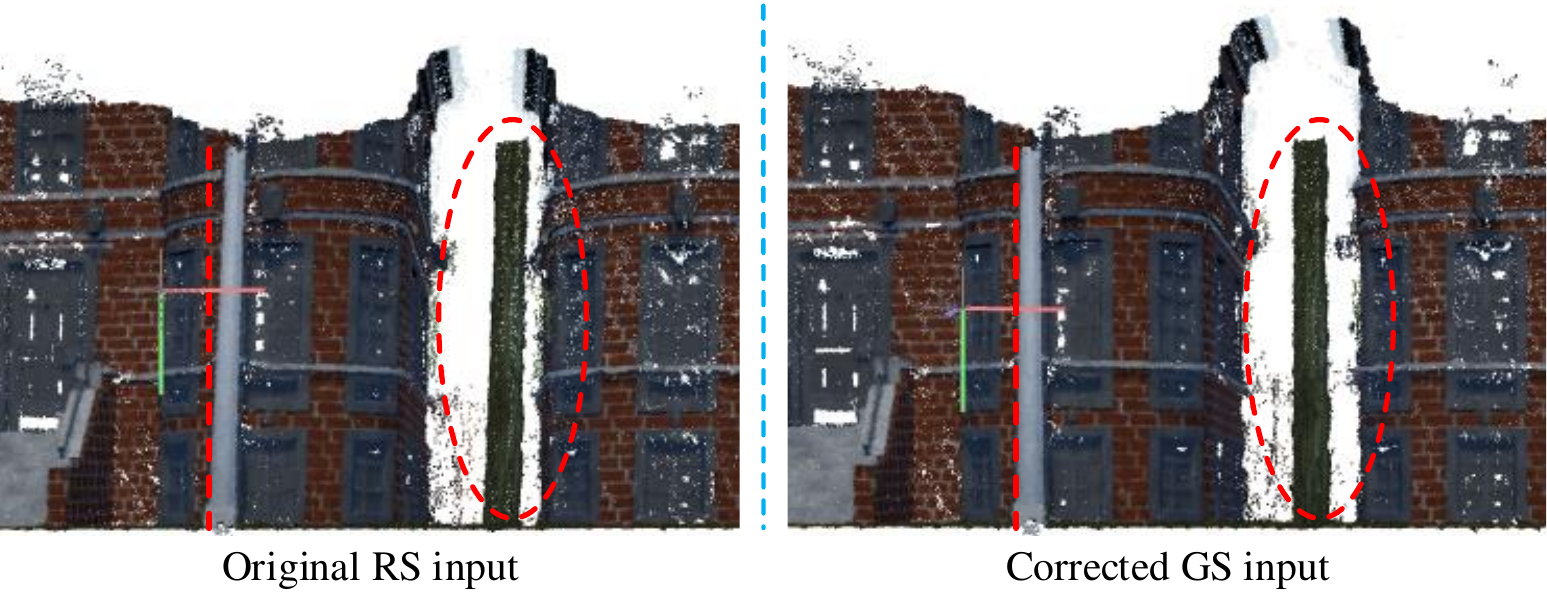}\vspace{-3.0mm}
	\caption{3D reconstruction results using the original RS images and our corrected GS images, respectively. Our method removes RS distortion and thus facilitates the recovery of more accurate 3D geometry. \label{fig:3D_Construction}
	}
	\vspace{-1.0mm}
\end{figure}

\subsection{3D Reconstruction Evaluation}
We run an SfM pipeline (\emph{i.e.}, Colmap \cite{schonberger2016structure}) to process the original RS images and our corrected GS images, respectively. Fig.~\ref{fig:3D_Construction} demonstrates that our method can remove the geometric inaccuracies and reconstruct a more reliable 3D scene structure. Note that significant geometric distortion appears in the 3D model obtained from the original RS images.
Overall, our method is beneficial to downstream applications with its superior ability to remove RS artifacts.



\subsection{Additional Comparison with DeepUnrollNet}\label{Analysis_DeepUnrollNet}
In our evaluation, in order to be consistent with \cite{zhuang2019learning,zhuang2020homography,zhuang2017rolling} (\emph{i.e.}, recovering the GS image corresponding to the first scanline), all competing results relate to the first scanline of the second frame. We thus retrain DeepUnrollNet \cite{liu2020deep} to adapt to this task for fair comparison. However, the original publication of \cite{liu2020deep} is designed to restore a GS image corresponding to the middle scanline of the second frame, so we additionally add quantitative and qualitative results under the middle scanline of the second frame, as shown in Table~\ref{t_s_1} and Fig.~\ref{fig:analysy_DeepUnrollNet}. One can see that, except for some occluded black edges (\emph{e.g.}, CRM score), our method is comparable or superior to DeepUnrollNet in returning the GS image corresponding to the middle scanline of the second frame, which verifies the validity of Proposition~\ref{p2}.
Moreover, combined with the experiments in Table~\ref{ablation_U_T} for learning a generalized middle-scanline undistortion flow (\emph{i.e.}, the ablation ``w/o P\ref{p2}''), this is why we follow the middle-scanline undistortion flow generated by scaling the corresponding optical flow vector under the constant velocity model for our subsequent constant acceleration propagation, \emph{i.e.}, the initial middle-scanline undistortion flow is able to yield a reliable and accurate prediction of the warping displacement direction.

Note that our method can restore the GS image corresponding to any scanline \emph{without demanding access to the supervision of the corresponding GT GS images}.
However, DeepUnrollNet and JCD \cite{zhong2021rscd} are limited to hallucinate a single middle-scanline GS frame.
Although DeepUnrollNet can also be trained to learn the mapping between the RS image and the first-scanline GS image when the costly first-scanline GT is provided, recovering the GS image corresponding to the first scanline of the second frame is not impressive.
Since acquiring the GT GS images is expensive and laborious, especially for arbitrary scanlines, it poses significant challenges for network design and model training.
Our method successfully circumvents this problem and is satisfactory in restoring GS images of the first scanlines, middle scanlines, and even arbitrary scanlines.

\begin{table}[!t]
	\footnotesize
	\caption{Quantitative comparisons of recovering GS images corresponding to the middle scanline of the second RS frame.
}\label{t_s_1}
	\centering
	\vspace{-2.5mm}
	\begin{tabular}{lccccccccc}
		\hline
		\multirow{2}{*}{Method} & \multicolumn{3}{c}{PSNR$\uparrow$ (dB)}                      &   & \multicolumn{2}{c}{SSIM$\uparrow$}  \\ \cline{2-4} \cline{6-7}
		& CRM     & CR        & FR      & & CR        & FR       \\ \hline
		JCD$^\dag$ \cite{zhong2021rscd} & - & - & 24.84 & & - & 0.78 \\ \hline
		DeepUnrollNet \cite{liu2020deep}  & 27.86 & \textbf{27.54} & \textbf{27.02} & & 0.83 & \textbf{0.83}   \\
		RSSR (Ours)                       & \textbf{29.36} & 26.57 & 25.01 & & \textbf{0.90} & \textbf{0.83} \\ \hline
	\end{tabular}
	\vspace{-0.5mm}
	\begin{tablenotes}
		\raggedleft
		\item{
			\small{$\dag$: \emph{results copied from \cite{zhong2021rscd}}}
		}
	\end{tablenotes}
	\vspace{-4.0mm}
\end{table}

\subsection{Inference Times}
Our method can simultaneously predict two middle-scanline GS images with VGA resolution ($640 \times 480$ pixels) in near real-time on an NVIDIA GeForce RTX 2080Ti GPU (average 0.12 seconds), which is faster than the 0.34 seconds of \cite{liu2020deep} to restore a single GS image. Further, we can extend to generate a GS image corresponding to any certain scanline via propagation, using an average runtime of 1.75 milliseconds, because only simple explicit matrix operations are required.
Therefore, our vanilla RSSR method can efficiently \textbf{\emph{produce 960 GS video frames in about 1.80 seconds}}, achieving temporal upsampling from 30 fps to 14400 (480$\times$30) fps.
In other words, quickly and extremely enlarging the frame rate of reversed GS video will be easy to reach.
Optionally, the lightweight Acceleration-Net will take an additional 1.17 milliseconds under the constant acceleration motion model, which can be negligible.
DiffSfM \cite{zhuang2017rolling}, however, takes about 467.26 seconds to recover a single GS image on an Intel Core i7-7700K CPU, which is a disadvantage for time-constrained tasks, such as real-time robotic visual localization.

\subsection{Limitations}
The main limitation of our approach is that it is not robust to heavy occlusions and moving objects.
Occlusion increases the difficulty of estimating the RS geometry, which leads to black holes in the synthesized images, degrading the visual experience. This problem might be alleviated by performing context aggregation in feature space (\emph{e.g.} \cite{fan2021sunet,niklaus2018context}) or image space (\emph{e.g.} \cite{xu2019quadratic,jiang2018super}). Moreover, moving objects have always been a great challenge for model-based RS correction methods, \emph{e.g.}, \cite{lao2020rolling,zhuang2017rolling,zhuang2020homography,zhuang2019learning,rengarajan2017unrolling}. In image regions specific to moving objects, the RS undistortion flow will not be well estimated due to violation of camera motion assumptions, resulting in ambiguous motion artifacts.

\section{Conclusion} \label{sec:conclusion}
In this paper, we tackle the challenging task of RS image inversion, \emph{i.e.}, converting consecutive RS images to high framerate GS videos. To this end, we have constructed and discussed the inherent interaction between bidirectional RS undistortion flow and regular optical flow. We present the first novel and intuitive RS temporal super-resolution framework that extracts a latent GS image sequence from two consecutive RS images, which is guided by the underlying geometric properties of the problem itself. Our pipeline owns good interpretability and generalization ability due to RS geometry-aware learning.
Furthermore, we generalize and develop a constant acceleration propagation method, which contributes to further improving the accuracy and robustness of the RSSR task.
We have demonstrated that our approach not only reconstructs geometrically and temporally consistent video sequences efficiently but also removes RS artifacts. In the future, we plan to extend our method to handling continuous rolling-shutter videos directly, and making better use of the scene context information. Meanwhile, we hope that our work will help incorporate traditional computer vision knowledge into deep learning systems, so as to harness decades of valuable research and contribute to the development of a more principled framework.

\ifCLASSOPTIONcompsoc
  \section*{Acknowledgments}
  This research was supported in part by National Natural Science Foundation of China (62271410, 62001394, and 61901387) and National Key Research and Development Program of China (2018AAA0102803).
\else
  \section*{Acknowledgment}
\fi


\ifCLASSOPTIONcaptionsoff
  \newpage
\fi



%

%
%

\bibliographystyle{IEEEtran}
\bibliography{RS_References}

\begin{thebibliography}{10}
\providecommand{\url}[1]{#1}
\csname url@samestyle\endcsname
\providecommand{\newblock}{\relax}
\providecommand{\bibinfo}[2]{#2}
\providecommand{\BIBentrySTDinterwordspacing}{\spaceskip=0pt\relax}
\providecommand{\BIBentryALTinterwordstretchfactor}{4}
\providecommand{\BIBentryALTinterwordspacing}{\spaceskip=\fontdimen2\font plus
\BIBentryALTinterwordstretchfactor\fontdimen3\font minus
  \fontdimen4\font\relax}
\providecommand{\BIBforeignlanguage}[2]{{%
\expandafter\ifx\csname l@#1\endcsname\relax
\typeout{** WARNING: IEEEtran.bst: No hyphenation pattern has been}%
\typeout{** loaded for the language `#1'. Using the pattern for}%
\typeout{** the default language instead.}%
\else
\language=\csname l@#1\endcsname
\fi
#2}}
\providecommand{\BIBdecl}{\relax}
\BIBdecl

\bibitem{dai2016rolling}
Y.~Dai, H.~Li, and L.~Kneip, ``Rolling shutter camera relative pose:
  generalized epipolar geometry,'' in \emph{Proceedings of IEEE/CVF Conference
  on Computer Vision and Pattern Recognition}, 2016, pp. 4132--4140.

\bibitem{lao2020rolling}
Y.~Lao and O.~Ait-Aider, ``Rolling shutter homography and its applications,''
  \emph{IEEE Transactions on Pattern Analysis and Machine Intelligence},
  vol.~43, no.~8, pp. 2780--2793, 2021.

\bibitem{albl2019rolling}
C.~Albl, Z.~Kukelova, V.~Larsson, and T.~Pajdla, ``Rolling shutter camera
  absolute pose,'' \emph{IEEE Transactions on Pattern Analysis and Machine
  Intelligence}, vol.~42, no.~6, pp. 1439--1452, 2019.

\bibitem{wang2021depth}
K.~Wang, C.~Liu, K.~Wang, and S.~Shen, ``Depth estimation under motion with
  single pair rolling shutter stereo images,'' \emph{IEEE Robotics and
  Automation Letters}, vol.~6, no.~2, pp. 3160--3167, 2021.

\bibitem{im2018accurate}
S.~Im, H.~Ha, G.~Choe, H.-G. Jeon, K.~Joo, and I.~S. Kweon, ``Accurate 3d
  reconstruction from small motion clip for rolling shutter cameras,''
  \emph{IEEE Transactions on Pattern Analysis and Machine Intelligence},
  vol.~41, no.~4, pp. 775--787, 2018.

\bibitem{rengarajan2017unrolling}
V.~Rengarajan, Y.~Balaji, and A.~Rajagopalan, ``Unrolling the shutter: cnn to
  correct motion distortions,'' in \emph{Proceedings of IEEE/CVF Conference on
  Computer Vision and Pattern Recognition}, 2017, pp. 2291--2299.

\bibitem{zhuang2019learning}
B.~Zhuang, Q.-H. Tran, P.~Ji, L.-F. Cheong, and M.~Chandraker, ``Learning
  structure-and-motion-aware rolling shutter correction,'' in \emph{Proceedings
  of IEEE/CVF Conference on Computer Vision and Pattern Recognition}, 2019, pp.
  4551--4560.

\bibitem{rengarajan2016bows}
V.~Rengarajan, A.~N. Rajagopalan, and R.~Aravind, ``From bows to arrows:
  rolling shutter rectification of urban scenes,'' in \emph{Proceedings of
  IEEE/CVF Conference on Computer Vision and Pattern Recognition}, 2016, pp.
  2773--2781.

\bibitem{lao2018robust}
Y.~Lao and O.~Ait-Aider, ``A robust method for strong rolling shutter effects
  correction using lines with automatic feature selection,'' in
  \emph{Proceedings of IEEE/CVF Conference on Computer Vision and Pattern
  Recognition}, 2018, pp. 4795--4803.

\bibitem{vasu2018occlusion}
S.~Vasu, M.~M. Mohan, and A.~Rajagopalan, ``Occlusion-aware rolling shutter
  rectification of 3d scenes,'' in \emph{Proceedings of IEEE/CVF Conference on
  Computer Vision and Pattern Recognition}, 2018, pp. 636--645.

\bibitem{zhong2021rscd}
Z.~Zhong, Y.~Zheng, and I.~Sato, ``Towards rolling shutter correction and
  deblurring in dynamic scenes,'' in \emph{Proceedings of IEEE/CVF Conference
  on Computer Vision and Pattern Recognition}, 2021, pp. 9219--9228.

\bibitem{hedborg2012rolling}
J.~Hedborg, P.-E. Forss{\'e}n, M.~Felsberg, and E.~Ringaby, ``Rolling shutter
  bundle adjustment,'' in \emph{Proceedings of IEEE/CVF Conference on Computer
  Vision and Pattern Recognition}, 2012, pp. 1434--1441.

\bibitem{fan2021rsstereo}
B.~Fan, Y.~Dai, and K.~Wang, ``Rolling-shutter-stereo-aware motion estimation
  and image correction,'' \emph{Computer Vision and Image Understanding}, vol.
  213, p. 103296, 2021.

\bibitem{liu2020deep}
P.~Liu, Z.~Cui, V.~Larsson, and M.~Pollefeys, ``Deep shutter unrolling
  network,'' in \emph{Proceedings of IEEE/CVF Conference on Computer Vision and
  Pattern Recognition}, 2020, pp. 5941--5949.

\bibitem{zhuang2017rolling}
B.~Zhuang, L.-F. Cheong, and G.~Hee~Lee, ``Rolling-shutter-aware differential
  sfm and image rectification,'' in \emph{Proceedings of IEEE International
  Conference on Computer Vision}, 2017, pp. 948--956.

\bibitem{zhuang2020homography}
B.~Zhuang and Q.-H. Tran, ``Image stitching and rectification for hand-held
  cameras,'' in \emph{Proceedings of European Conference on Computer Vision},
  2020, pp. 243--260.

\bibitem{fan2021sunet}
B.~Fan, Y.~Dai, and M.~He, ``Sunet: symmetric undistortion network for rolling
  shutter correction,'' in \emph{Proceedings of IEEE International Conference
  on Computer Vision}, 2021, pp. 4541--4550.

\bibitem{fan2021rssr}
B.~Fan and Y.~Dai, ``Inverting a rolling shutter camera: bring rolling shutter
  images to high framerate global shutter video,'' in \emph{Proceedings of IEEE
  International Conference on Computer Vision}, 2021, pp. 4228--4237.

\bibitem{jiang2018super}
H.~Jiang, D.~Sun, V.~Jampani, M.-H. Yang, E.~Learned-Miller, and J.~Kautz,
  ``Super slomo: high quality estimation of multiple intermediate frames for
  video interpolation,'' in \emph{Proceedings of IEEE/CVF Conference on
  Computer Vision and Pattern Recognition}, 2018, pp. 9000--9008.

\bibitem{niklaus2020softmax}
S.~Niklaus and F.~Liu, ``Softmax splatting for video frame interpolation,'' in
  \emph{Proceedings of IEEE/CVF Conference on Computer Vision and Pattern
  Recognition}, 2020, pp. 5437--5446.

\bibitem{bao2019depth}
W.~Bao, W.-S. Lai, C.~Ma, X.~Zhang, Z.~Gao, and M.-H. Yang, ``Depth-aware video
  frame interpolation,'' in \emph{Proceedings of IEEE/CVF Conference on
  Computer Vision and Pattern Recognition}, 2019, pp. 3703--3712.

\bibitem{sun2018pwc}
D.~Sun, X.~Yang, M.-Y. Liu, and J.~Kautz, ``Pwc-net: cnns for optical flow
  using pyramid, warping, and cost volume,'' in \emph{Proceedings of IEEE/CVF
  Conference on Computer Vision and Pattern Recognition}, 2018, pp. 8934--8943.

\bibitem{meyer2015phase}
S.~Meyer, O.~Wang, H.~Zimmer, M.~Grosse, and A.~Sorkine-Hornung, ``Phase-based
  frame interpolation for video,'' in \emph{Proceedings of IEEE/CVF Conference
  on Computer Vision and Pattern Recognition}, 2015, pp. 1410--1418.

\bibitem{meyer2018phasenet}
S.~Meyer, A.~Djelouah, B.~McWilliams, A.~Sorkine-Hornung, M.~Gross, and
  C.~Schroers, ``Phasenet for video frame interpolation,'' in \emph{Proceedings
  of IEEE/CVF Conference on Computer Vision and Pattern Recognition}, 2018, pp.
  498--507.

\bibitem{choi2020channel}
M.~Choi, H.~Kim, B.~Han, N.~Xu, and K.~M. Lee, ``Channel attention is all you
  need for video frame interpolation,'' in \emph{Proceedings of the AAAI
  Conference on Artificial Intelligence}, vol.~34, no.~07, 2020, pp.
  10\,663--10\,671.

\bibitem{liu2017video}
Z.~Liu, R.~A. Yeh, X.~Tang, Y.~Liu, and A.~Agarwala, ``Video frame synthesis
  using deep voxel flow,'' in \emph{Proceedings of IEEE International
  Conference on Computer Vision}, 2017, pp. 4463--4471.

\bibitem{park2020bmbc}
J.~Park, K.~Ko, C.~Lee, and C.-S. Kim, ``Bmbc: bilateral motion estimation with
  bilateral cost volume for video interpolation,'' in \emph{Proceedings of
  European Conference on Computer Vision}, 2020, pp. 109--125.

\bibitem{teed2020raft}
Z.~Teed and J.~Deng, ``Raft: recurrent all-pairs field transforms for optical
  flow,'' in \emph{Proceedings of European Conference on Computer Vision},
  2020, pp. 402--419.

\bibitem{niklaus2018context}
S.~Niklaus and F.~Liu, ``Context-aware synthesis for video frame
  interpolation,'' in \emph{Proceedings of IEEE/CVF Conference on Computer
  Vision and Pattern Recognition}, 2018, pp. 1701--1710.

\bibitem{xu2019quadratic}
X.~Xu, L.~Siyao, W.~Sun, Q.~Yin, and M.-H. Yang, ``Quadratic video
  interpolation,'' in \emph{Proceedings of Advances in Neural Information
  Processing Systems}, vol.~32, 2019.

\bibitem{liu2020enhanced}
Y.~Liu, L.~Xie, L.~Siyao, W.~Sun, Y.~Qiao, and C.~Dong, ``Enhanced quadratic
  video interpolation,'' in \emph{Proceedings of European Conference on
  Computer Vision}, 2020, pp. 41--56.

\bibitem{chi2020all}
Z.~Chi, R.~Mohammadi~Nasiri, Z.~Liu, J.~Lu, J.~Tang, and K.~N. Plataniotis,
  ``All at once: temporally adaptive multi-frame interpolation with advanced
  motion modeling,'' in \emph{Proceedings of European Conference on Computer
  Vision}, 2020, pp. 107--123.

\bibitem{xue2019video}
T.~Xue, B.~Chen, J.~Wu, D.~Wei, and W.~T. Freeman, ``Video enhancement with
  task-oriented flow,'' \emph{International Journal of Computer Vision}, vol.
  127, no.~8, pp. 1106--1125, 2019.

\bibitem{reda2019unsupervised}
F.~A. Reda, D.~Sun, A.~Dundar, M.~Shoeybi, G.~Liu, K.~J. Shih, A.~Tao,
  J.~Kautz, and B.~Catanzaro, ``Unsupervised video interpolation using cycle
  consistency,'' in \emph{Proceedings of IEEE International Conference on
  Computer Vision}, 2019, pp. 892--900.

\bibitem{liu2019deep}
Y.-L. Liu, Y.-T. Liao, Y.-Y. Lin, and Y.-Y. Chuang, ``Deep video frame
  interpolation using cyclic frame generation,'' in \emph{Proceedings of the
  AAAI Conference on Artificial Intelligence}, vol.~33, no.~01, 2019, pp.
  8794--8802.

\bibitem{shen2020blurry}
W.~Shen, W.~Bao, G.~Zhai, L.~Chen, X.~Min, and Z.~Gao, ``Blurry video frame
  interpolation,'' in \emph{Proceedings of IEEE/CVF Conference on Computer
  Vision and Pattern Recognition}, 2020, pp. 5114--5123.

\bibitem{forssen2010rectifying}
P.-E. Forss{\'e}n and E.~Ringaby, ``Rectifying rolling shutter video from
  hand-held devices,'' in \emph{Proceedings of IEEE/CVF Conference on Computer
  Vision and Pattern Recognition}, 2010, pp. 507--514.

\bibitem{saurer2013rolling}
O.~Saurer, K.~Koser, J.-Y. Bouguet, and M.~Pollefeys, ``Rolling shutter
  stereo,'' in \emph{Proceedings of IEEE International Conference on Computer
  Vision}, 2013, pp. 465--472.

\bibitem{wang2020relative}
K.~Wang, B.~Fan, and Y.~Dai, ``Relative pose estimation for stereo rolling
  shutter cameras,'' in \emph{Proceedings of IEEE International Conference on
  Image Processing}, 2020, pp. 463--467.

\bibitem{fan2021rsdpsnet}
B.~Fan, K.~Wang, Y.~Dai, and M.~He, ``Rs-dpsnet: deep plane sweep network for
  rolling shutter stereo images,'' \emph{IEEE Signal Processing Letters},
  vol.~28, pp. 1550--1554, 2021.

\bibitem{purkait2017rolling}
P.~Purkait, C.~Zach, and A.~Leonardis, ``Rolling shutter correction in
  {Manhattan} world,'' in \emph{Proceedings of IEEE International Conference on
  Computer Vision}, 2017, pp. 882--890.

\bibitem{lao2021solving}
Y.~Lao, O.~Ait-Aider, and A.~Bartoli, ``Solving rolling shutter 3d vision
  problems using analogies with non-rigidity,'' \emph{International Journal of
  Computer Vision}, vol. 129, no.~1, pp. 100--122, 2021.

\bibitem{grundmann2012calibration}
M.~Grundmann, V.~Kwatra, D.~Castro, and I.~Essa, ``Calibration-free rolling
  shutter removal,'' in \emph{Proceedings of IEEE International Conference on
  Computational Photography}, 2012, pp. 1--8.

\bibitem{albl2020two}
C.~Albl, Z.~Kukelova, V.~Larsson, M.~Polic, T.~Pajdla, and K.~Schindler, ``From
  two rolling shutters to one global shutter,'' in \emph{Proceedings of
  IEEE/CVF Conference on Computer Vision and Pattern Recognition}, 2020, pp.
  2505--2513.

\bibitem{wu2021simultaneous}
H.~Wu, L.~Xiao, and Z.~Wei, ``Simultaneous video stabilization and rolling
  shutter removal,'' \emph{IEEE Transactions on Image Processing}, vol.~30, pp.
  4637--4652, 2021.

\bibitem{longuet1980interpretation}
H.~C. Longuet-Higgins and K.~Prazdny, ``The interpretation of a moving retinal
  image,'' \emph{Proceedings of the Royal Society of London. Series B.
  Biological Sciences}, vol. 208, no. 1173, pp. 385--397, 1980.

\bibitem{patron2015spline}
A.~Patron-Perez, S.~Lovegrove, and G.~Sibley, ``A spline-based trajectory
  representation for sensor fusion and rolling shutter cameras,''
  \emph{International Journal of Computer Vision}, vol. 113, no.~3, pp.
  208--219, 2015.

\bibitem{meingast2005geometric}
M.~Meingast, C.~Geyer, and S.~Sastry, ``Geometric models of rolling-shutter
  cameras,'' \emph{arXiv preprint arXiv:cs/0503076}, 2005.

\bibitem{oth2013rolling}
L.~Oth, P.~Furgale, L.~Kneip, and R.~Siegwart, ``Rolling shutter camera
  calibration,'' in \emph{Proceedings of IEEE/CVF Conference on Computer Vision
  and Pattern Recognition}, 2013, pp. 1360--1367.

\bibitem{sun2010secrets}
D.~Sun, S.~Roth, and M.~J. Black, ``Secrets of optical flow estimation and
  their principles,'' in \emph{Proceedings of IEEE/CVF Conference on Computer
  Vision and Pattern Recognition}, 2010, pp. 2432--2439.

\bibitem{fischler1981random}
M.~A. Fischler and R.~C. Bolles, ``Random sample consensus: a paradigm for
  model fitting with applications to image analysis and automated
  cartography,'' \emph{Communications of the ACM}, vol.~24, no.~6, pp.
  381--395, 1981.

\bibitem{ringaby2012efficient}
E.~Ringaby and P.-E. Forss{\'e}n, ``Efficient video rectification and
  stabilisation for cell-phones,'' \emph{International Journal of Computer
  Vision}, vol.~96, no.~3, pp. 335--352, 2012.

\bibitem{wang2018occlusion}
Y.~Wang, Y.~Yang, Z.~Yang, L.~Zhao, P.~Wang, and W.~Xu, ``Occlusion aware
  unsupervised learning of optical flow,'' in \emph{Proceedings of IEEE/CVF
  Conference on Computer Vision and Pattern Recognition}, 2018, pp. 4884--4893.

\bibitem{liu2020self}
P.~Liu, J.~Janai, M.~Pollefeys, T.~Sattler, and A.~Geiger, ``Self-supervised
  linear motion deblurring,'' \emph{IEEE Robotics and Automation Letters},
  vol.~5, no.~2, pp. 2475--2482, 2020.

\bibitem{ronneberger2015unet}
O.~Ronneberger, P.~Fischer, and T.~Brox, ``U-net: convolutional networks for
  biomedical image segmentation,'' in \emph{Proceedings of the International
  Conference on Medical Image Computing and Computer-assisted Intervention},
  2015, pp. 234--241.

\bibitem{johnson2016perceptual}
J.~Johnson, A.~Alahi, and L.~Fei-Fei, ``Perceptual losses for real-time style
  transfer and super-resolution,'' in \emph{Proceedings of European Conference
  on Computer Vision}, 2016, pp. 694--711.

\bibitem{simonyan2014very}
K.~Simonyan and A.~Zisserman, ``Very deep convolutional networks for
  large-scale image recognition,'' in \emph{Proceedings of the International
  Conference on Learning Representations}, 2015.

\bibitem{dosovitskiy2017carla}
A.~Dosovitskiy, G.~Ros, F.~Codevilla, A.~Lopez, and V.~Koltun, ``Carla: an open
  urban driving simulator,'' in \emph{Proceedings of the 1st Annual Conference
  on Robot Learning}, 2017, pp. 1--16.

\bibitem{Kingma_Adam_ICLR_2015}
D.~P. Kingma and J.~Ba, ``Adam: a method for stochastic optimization,'' in
  \emph{Proceedings of the International Conference on Learning
  Representations}, 2015.

\bibitem{zhang2018unreasonable}
R.~Zhang, P.~Isola, A.~A. Efros, E.~Shechtman, and O.~Wang, ``The unreasonable
  effectiveness of deep features as a perceptual metric,'' in \emph{Proceedings
  of IEEE/CVF Conference on Computer Vision and Pattern Recognition}, 2018, pp.
  586--595.

\bibitem{schonberger2016structure}
J.~L. Schonberger and J.-M. Frahm, ``Structure-from-motion revisited,'' in
  \emph{Proceedings of IEEE/CVF Conference on Computer Vision and Pattern
  Recognition}, 2016, pp. 4104--4113.

\end{thebibliography}

%

\vspace{-10mm}
\begin{IEEEbiography}[{\includegraphics[width=1in,height=1.25in,clip,keepaspectratio]{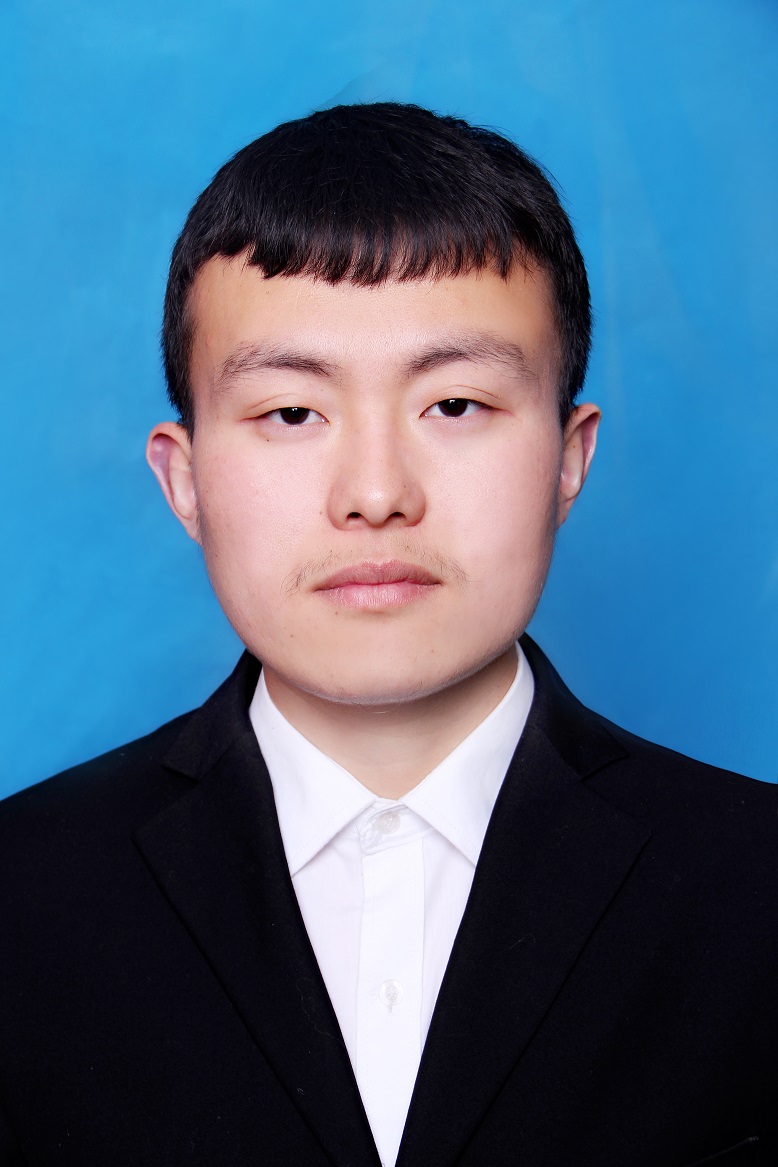}}]{Bin Fan} is currently a PhD student with School of Electronics and Information, Northwestern Polytechnical University (NPU), Xi'an, China.
He received the B.S. degree in Statistics and the M.E. degree in Control Science and Engineering from NPU, in 2016 and 2019, respectively.
He was selected to CVPR 2022 Doctoral Consortium (the only one among Chinese universities).
He co-organized the ACCV 2022 tutorial on the topic of rolling shutter cameras.
His research interests include computer vision, computational photography, 3D reconstruction, and deep learning.
\end{IEEEbiography}

\vspace{-10mm}
\begin{IEEEbiography}[{\includegraphics[width=1in,height=1.25in,clip,keepaspectratio]{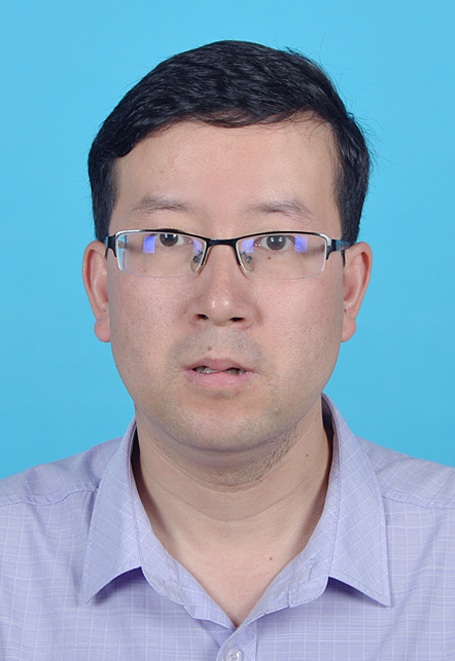}}]{Yuchao Dai} is currently a Professor with School of Electronics and Information at the Northwestern Polytechnical University (NPU). He received the B.E. degree, M.E degree and Ph.D. degree all in signal and information processing from Northwestern Polytechnical University, Xi'an, China, in 2005, 2008 and 2012, respectively. He was an ARC DECRA Fellow with the Research School of Engineering at the Australian National University, Canberra, Australia. His research interests include structure from motion, multi-view geometry, low-level computer vision, deep learning, compressive sensing and optimization. He won the Best Paper Award in IEEE CVPR 2012, the DSTO Best Fundamental Contribution to Image Processing Paper Prize at DICTA 2014, the Best Algorithm Prize in NRSFM Challenge at CVPR 2017, the Best Student Paper Prize at DICTA 2017, the Best Deep/Machine Learning Paper Prize at APSIPA ASC 2017, the Best Paper Award Nominee at IEEE CVPR 2020. He served as Area Chair in CVPR, ICCV, ACM MM, ACCV, WACV, etc. He serves as Publicity Chair in ACCV 2022.
\end{IEEEbiography}

\vspace{-10mm}
\begin{IEEEbiography}[{\includegraphics[width=1in,height=1.25in,clip,keepaspectratio]{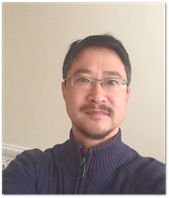}}]{Hongdong Li} is a Chief Investigator of the Australian Centre of Excellence for Robotic Vision, professor of Computer Science with the Australian National University. He joined the RSISE of ANU from 2004. He was a Visiting Professor with the Robotics Institutes, Carnegie Mellon University, doing a sabbatical in 2017-2018. His research interests include 3D computer vision, machine learning, autonomous driving, virtual and augmented reality, and mathematical optimization. He is an Associate Editor for IEEE Transactions on PAMI, and IVC, and served as Area Chair for recent CVPR, ICCV and ECCV. He was the winner of the 2012 CVPR Best Paper Award, the 2017 Marr Prize (Honorable Mention), a finalist for the CVPR 2020 best paper award, ICPR Best student paper award, ICIP Best student paper award, DSTO Fundamental Contribution to Image Processing Prize at DICTA 2014, Best algorithm award in CVPR NRSFM Challenge 2017, and a Best Practice Paper (honourable mention) award at WACV 2020. He is a Co Program Chair for ACCV 2018 and co general chair for ACCV 2022, Co-Publication Chair for IEEE ICCV 2019.
\end{IEEEbiography}

\end{document}